%% file: iclr2026_conference.tex
\documentclass{article} 
\usepackage{iclr2024_conference,times}

\input{math_commands.tex}

\usepackage[utf8]{inputenc} 
\usepackage[T1]{fontenc}    
\usepackage{hyperref}       
\usepackage{url}            
\usepackage{booktabs}       
\usepackage{enumitem}
\usepackage{amsfonts}       
\usepackage{nicefrac}       
\usepackage{microtype}      
\usepackage{xcolor}         

\usepackage{wrapfig}
\usepackage{amsthm}
\usepackage{amsmath}
\usepackage{amssymb}
\usepackage{float}
\usepackage{graphicx}
\usepackage{dsfont}
\usepackage{algorithm}
\usepackage{algpseudocode}
\usepackage{bm}

\usepackage{subfigure}
\usepackage[T1]{fontenc}

\newcommand{\bs}[1]{\boldsymbol{\mathbf{#1}}}
\newcommand{\bbE}{\mathbb{E}}

\newcommand{\bbR}{\mathbb{R}}

\newtheorem{theorem}{Theorem}
\newtheorem{assumption}[theorem]{Assumption}

\newcommand\nnfootnote[1]{%
  \begin{NoHyper}
  \renewcommand\thefootnote{}\footnote{#1}%
  \addtocounter{footnote}{-1}%
  \end{NoHyper}
}

\title{AdaPM: a Partial Momentum Algorithm for  LLM Training}

\author{Yimu Zhang$ ^{*1}$, Yuanshi Liu$^{*1}$, Cong Fang$^{1,2}$\\
$^{1}$ State Key Lab of General AI, School of Intelligence Science and Technology, Peking University \\
$^{2}$ Institute for Artificial Intelligence, Peking University \\
}
\iclrfinalcopy 
\begin{document}

\maketitle
\nnfootnote{$*$: Equal contribution. Correspondence to \texttt{fangcong@pku.edu.cn}.}

\begin{abstract}
In the training of large language models, momentum is widely used and often demonstrated to achieve significant acceleration. However, storing momentum typically presents memory challenges. In this paper, we propose AdaPM, an adaptive training strategy that leverages partial momentum to implement a memory-efficient optimizer. To this end, AdaPM utilizes a non-uniform momentum design: for most blocks, full momentum is not necessary to preserve the performance of the optimization. In the momentum design of AdaPM, to mitigate the bias and performance loss caused by partial momentum, we enhance the partial momentum by a bias correction technique. Empirically, we verify that our approach reduces memory by over $90\%$ in momentum while maintaining both efficiency and performance for pretraining various language models ranging from 60M to 1.5B, as well as for supervised fine-tuning and RLHF. AdaPM can further reduce memory by up to $95\%$ in optimizer states by combining the memory-efficient technique on the second-order statistic, saving over $30\%$ GPU hours for pretraining GPT-2 1.5B.\footnote{Our implementation for AdaPM will be released at github.}
\end{abstract}

\section{Introduction}\label{sec:intro}
Efficient optimizers, working as engines,  are a key factor in the success and boom of modern large language models (LLMs) \citep{vaswani2017attention,achiam2023gpt,touvron2023llama,liu2024deepseek}. However, besides accelerating training, the optimizers also introduce significant memory overhead, posing a major challenge to limited memory resources. A typical example is the widely-used Adam optimizer \citep{kingma2014adam}, which requires two additional sets of values—the first- and second-order statistics estimators for every parameter—thereby significantly increasing the demand for device memory. As model sizes increase, model performance continues to improve, yet the memory occupied by optimizer states alone becomes dominant \citep{zhao2024galore}, creating substantial implementation and time-consuming challenges.


The challenges have spurred significant interest in designing memory-efficient optimizers.  The pursuit of these optimizers is driven by  dual benefits \citep{zhao2024galore}: they circumvent the limitations on model size and enable larger batch sizes during parallel training, thereby lowering communication overhead,  and directly speeding up optimization.

Current approaches to reducing optimizer states primarily focus on second-order statistics, with comparatively less progress made in improving first-order statistics. This asymmetry stems from their distinct roles. Second-order statistics are non-negative and serve as per-parameter scale estimators. \citet{zhangadam} shows that they typically exhibit near-uniform scales within architectural units such as blocks or neurons. This reducibility has motivated methods like Adafactor \citep{shazeer2018adafactor} and Adam-mini \citep{zhangadam}, which significantly compress second-order statistics.

In contrast, first-order statistics---known as momentum in optimization theory---are widely recognized for accelerating convergence \citep{sutskever2013importance}. However, they are signed and highly sensitive in governing function value descent \citep{kunstner2023noise, fu2023and}. Attempts to reduce first-order statistics, such as the low-rank update algorithms  \citep{cosson2023low, zhao2024galore} often lead to non-negligible performance degradation during pretraining. Consequently, whether first-order statistics exhibit similar reducibility as second-order states and can be substantially compressed without compromising performance remains an open and important question.

We propose the method,  \textbf{Ada}ptive \textbf{P}artial \textbf{M}omentum, AdaPM,  which significantly reduces the memory footprint of first-order statistics without performance degradation at the \emph{first} time. One first insight of our method originates from the recent analyses that break down the integrated transformer architecture and reveal the heterogeneity across its components \citep{zhang2024transformers}. We find that most blocks do not require full
momentum acceleration. The second insight stems from the observation that momentum is only approximately low-rank. Its singular values are skewed: a few large singular values followed by a long tail of small yet also consequential ones (see Fig. \ref{fig:emp insight}). And truncating the tail, as in \citet{zhao2024galore}, slows convergence. To address this, we introduce a \emph{novel} residual compensation technique that corrects the discrepancy between the full momentum and its low-rank approximation. This technique restores discarded descent directions by rescaling the residual between the full momentum and its current approximation (see details in Section \ref{sec: debiased}), thereby mitigating performance degradation caused by the approximation. Furthermore, our method can be combined with approaches that compress second-order statistics, such as Adam-mini \citep{zhangadam}, enabling a substantial reduction in the overall memory footprint of optimizer states.


\vspace{-4pt}
\begin{figure}[tb]
    \centering
    \subfigure[ Loss v.s. iteration]{
    \begin{minipage}{0.32\linewidth}
    \includegraphics[width=47mm]{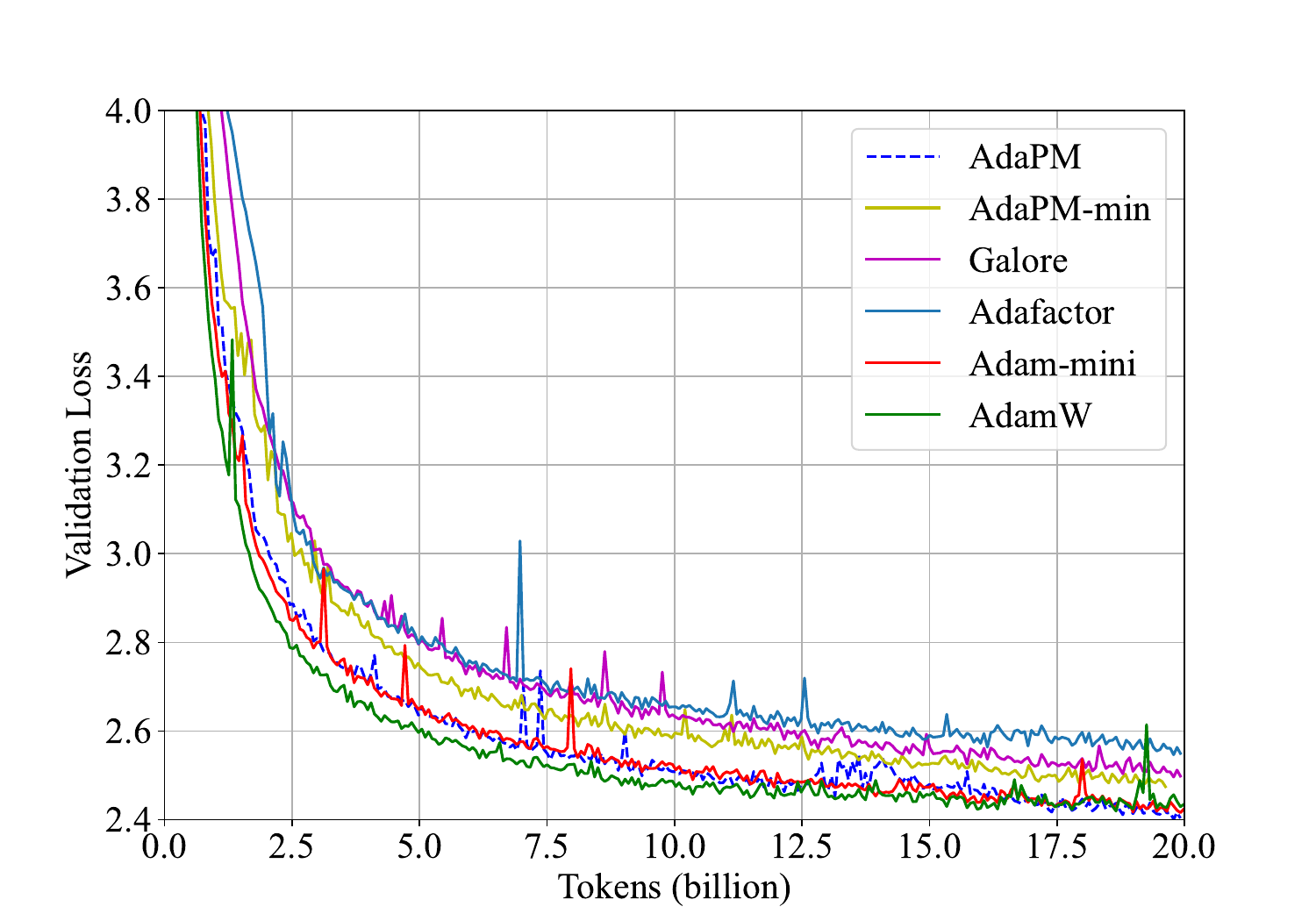}
    \end{minipage}
    }
    \subfigure[Memory cost]{
    \begin{minipage}{0.32\linewidth}
    \includegraphics[width=46mm]{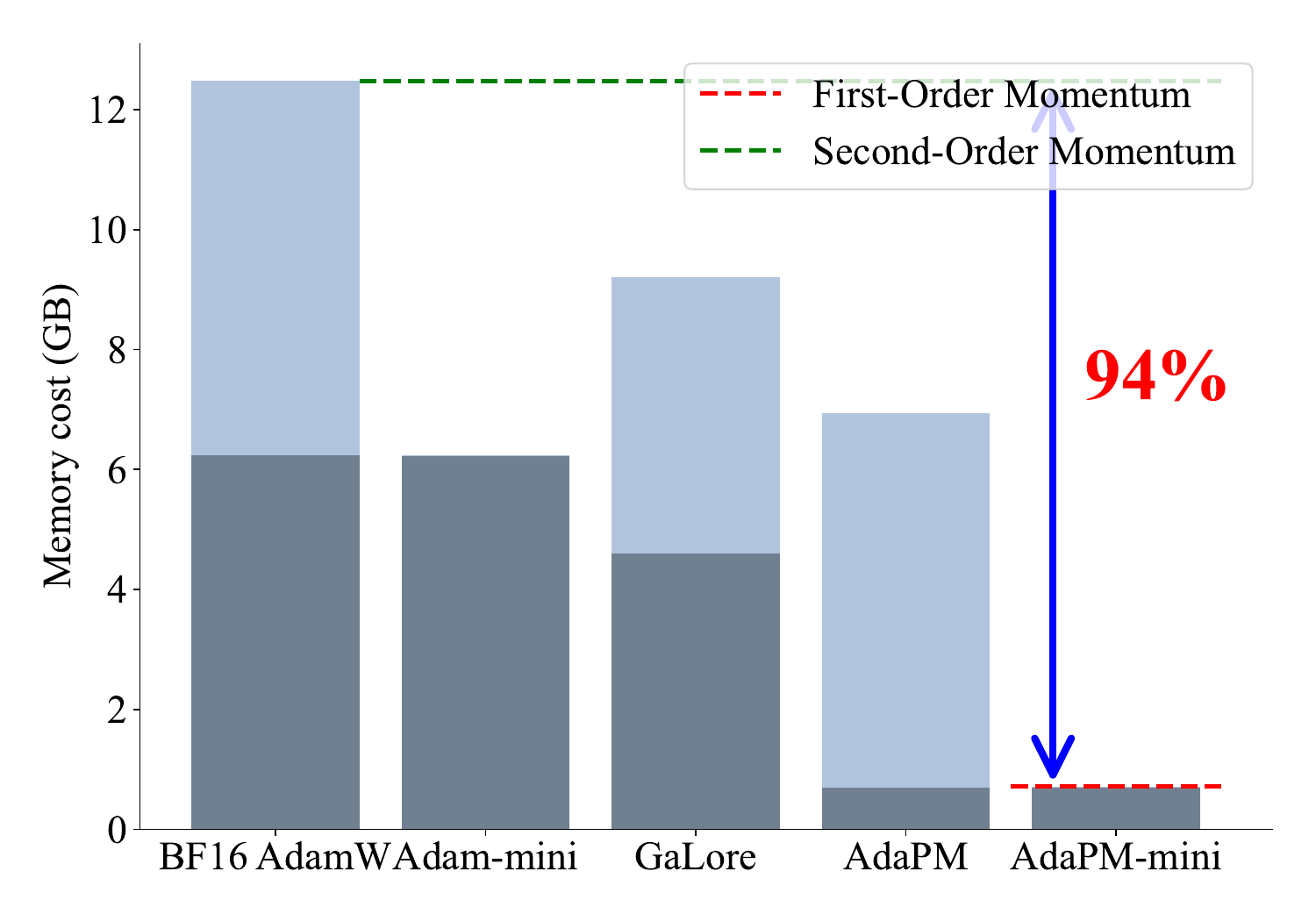}
    \end{minipage}
    }
     \subfigure[ An illustration of AdaPM]{
    \begin{minipage}{0.31\linewidth}
    \includegraphics[width=47mm]{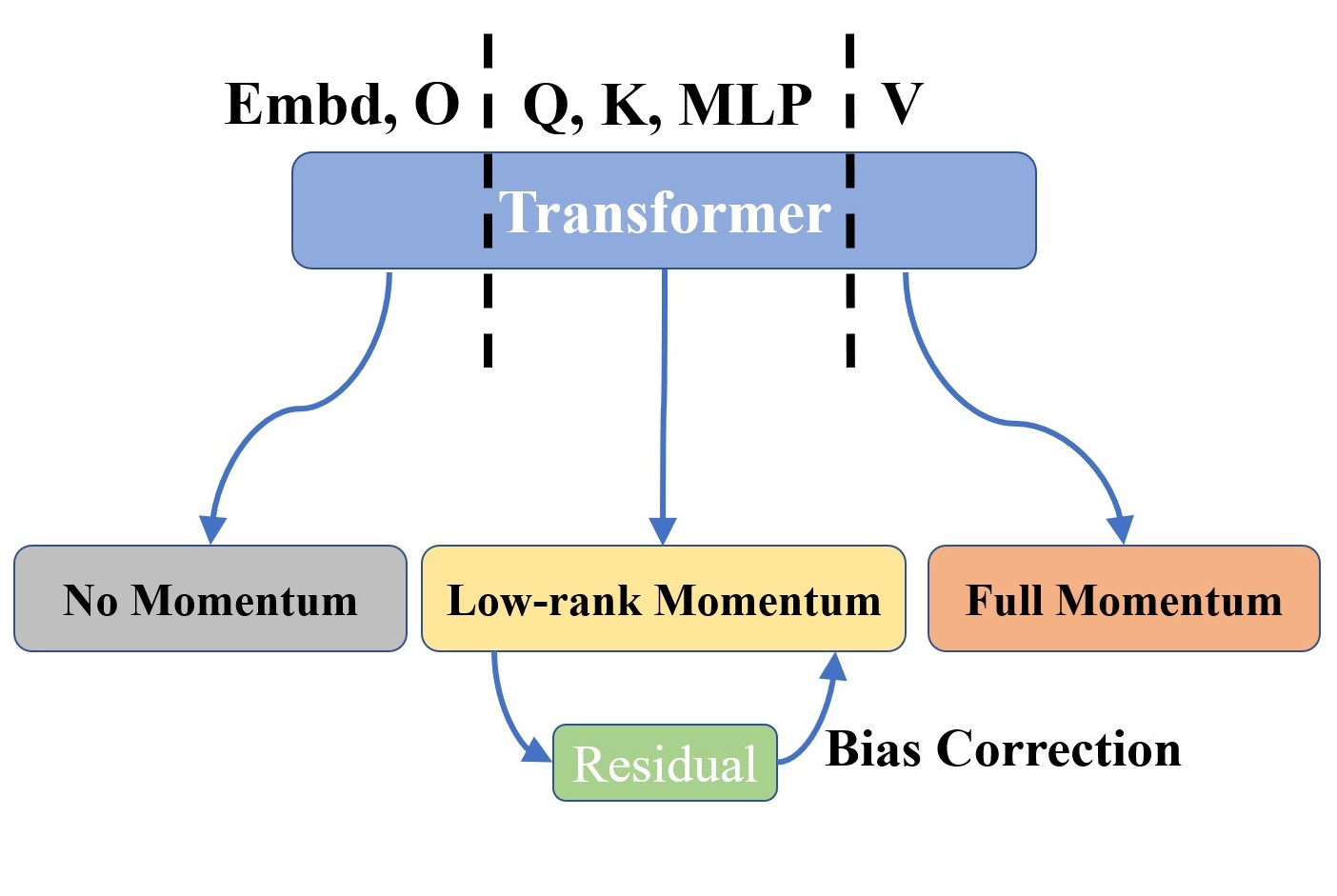}
    \end{minipage}
    }
    \caption{AdaPM takes less memory and can reach higher
 throughput with on par or better performance than AdamW.  (a)  Results for GPT-2 1.5B pre-training.  (b) The memory cost when training GPT-2 1.5B with various optimizers. The experimental details are shown in Section \ref{sec:ex-pre}. (c) AdaPM assigns different momentum
designs to different blocks and enhances the partial momentum using a bias-
corrected approach.}
    \label{fig: mem}
\end{figure}

Our experiments demonstrate that AdaPM matches or exceeds the performance of AdamW across various task with only $10\%$ of the momentum memory is required. We pre-train AdaPM under a GPT-2 series  and LLaMA series from 124M to 1.5B and also test AdaPM over Llama-3 8B post-training tasks. AdaPM achieves a consistent comparable or better convergence performance, showcasing strong scalability and robustness across model sizes. Notably, as shown in Fig \ref{fig: mem}, when combined Adam-mini, AdaPM can even achieves a remarkable $94\%$ memory saving on the optimizer state without sacrificing convergence of AdamW. Moreover, due to higher throughput and larger batch sizes resulting from memory reduction, AdaPM saving over 30$\%$ GPU hours for pretraining.

To summarize our contribution:
\begin{itemize}[leftmargin=3.3em, labelsep=1em] 
    \item Partition approach of Transformers. We investigate the impact of various blocks in the Transformer on momentum acceleration. We find that most blocks do not require full momentum acceleration: (1) Embedding and attn.proj blocks do not need momentum acceleration; (2) query, key and MLP need low rank momentum; (3) value needs full momentum.
    \item Bias-corrected  estimation of low-rank momentum. We propose a novel debiasing method that achieves unbiased estimation of full momentum, which only requires the low-rank momentum with merely $5\% $ of the original dimensions.
    \item Memory-efficient optimizer. Our work integrates the above partition principle with debiased low-rank momentum estimation, proposing a novel algorithmic framework that significantly reduces memory overhead in Adam optimization. By combining with methods like Adam-mini, this framework achieves over 95\% memory savings in the optimizer state without performance degradation.
\end{itemize}

\section{Related works}
\textbf{Lightweight optimizers focusing on second-order statistics}. Several lightweight optimization algorithms have been developed to reduce computational and memory costs by leveraging second-order statistical information. Adafactor \citep{shazeer2018adafactor} and CAME \citep{luo2023came} are memory-efficient variants of Adam employing nonnegative low-rank factorization over Adam’s second-order statistics. 
SM3 \citep{anil2019memory} employs a candidate set, which is derived from the maximum squared gradient within a specific group of parameters defined by a cover of the parameters, and determines the learning rate for the $i$-th parameter by selecting the smallest value from the candidate set. Adam-mini \citep{zhangadam}  partitions the parameters into blocks, a proposed principle on Hessian structure, and assigns a uniform learning rate for each block using the average of Adam’s second-order statistics in that block.

\textbf{Lightweight optimizers focusing on subspace learning}. Recent research has shown that the learning process primarily takes place within a lower-dimensional subspace of the parameter space \citep{gur2018gradient}. Research, such as \citet{9053036,yang2023spectral}, has applied the low-rank property of gradients during the training of neural networks to reduce the memory footprint during training. A similar approach has been widely used in meta-learning and continual learning \citep{lee2018gradient, chaudhry2020continual}.  GaLore \citep{zhao2024galore} and a novel variant, Golore \citep{he2024subspace}, calculate a low-rank gradient estimator and then calculate the first- and second-order statistics on this low-rank gradient estimator.

\section{Starting Point of Momentum Reduction}
\label{sec:starting point}

In this section, we present a detailed discussion on the potential reducibility of the momentum. We present the theoretical justification and empirical evidence in Section~\ref{sec:learningtheory} and Section~\ref{sec:empiricalevidence}, respectively.

\subsection{Theoretical Justification}\label{sec:learningtheory}

Our first illustration of reducibility in momentum is posed by the following fundamental question: Does adding momentum in optimization consistently lower the validation loss? To investigate, we adopt the standard framework of minimizing the validation risk $\mathcal{R}(\mathbf{W}) = \mathbb{E}_{\mathbf{x},y\sim\mathcal{D}}\ell(\mathbf{W};(\mathbf{x},y))$. Let $\mathbf{W}^*\in\arg\min_{\mathbf{W}}\mathcal{R}(\mathbf{W})$ denote the oracle minimizer, and let $\hat{\mathbf{W}}_{\text{opt}}$ denote the output of the optimization algorithm with stochastic gradients.



The distribution of the outputs produced by an optimization algorithm is comprised of the expectation of the output and how randomness drives the algorithm to deviate from that. This division naturally yields the following decomposition of the validation loss:
\begin{align}\label{eq:decomposition}
    \mathcal{R}(\hat{\mathbf{W}}_{\text{opt}}) - \mathcal{R}(\mathbf{W}^*) = \underbrace{\mathcal{R}(\hat{\mathbf{W}}_{\text{opt}}) - \mathcal{R}(\bar{\mathbf{W}}_{\text{opt}})}_{\text{term}\ \mathcal{A}} + \underbrace{\mathcal{R}(\bar{\mathbf{W}}_{\text{opt}}) - \mathcal{R}(\mathbf{W}^*)}_{\text{term} \ \mathcal{B}},
\end{align}
where $\bar{\mathbf{W}}_{\text{opt}} = \mathbb{E}\left[\hat{\mathbf{W}}_\text{opt}\right]$ is expectation of the algorithm’s output. 

Introducing momentum into deterministic problems is known to accelerate optimization \citep{polyak1964some,nesterov1983method}, suggesting that the term $\mathcal{B}$ can be optimized more efficiently with its inclusion. In contrast, for the term $\mathcal{A}$, the injected noise, momentum does not necessarily mitigate them; indeed, theoretical analyses indicate that it may even amplify the variance of $\hat{\mathbf{W}}_\text{opt}$.

We illustrate the effect of adding momentum through the generic high-dimensional linear regression problem, regressing the Gaussian covariate $\mathbf{x}\in\mathbb{R}^d$ following $\mathcal{N}(\mathbf{0}, \bf\Sigma)$ to response $y = \langle\mathbf{W}^*, \mathbf{x}\rangle + \epsilon \in\mathbb{R}$ with $\epsilon\sim\mathcal{N}(0,\sigma^2)$: $\min_{\mathbf{W}\in\mathbb{R}^d} (\langle\mathbf{W}, \mathbf{x}\rangle - y)^2$. We adopt this setting since it is both fundamental and representative: many insights into modern optimizers originate from quadratic analyses \citep{zhangadam,liu2023sophia},  and neural tangent kernel theory links such problems to the training dynamics of large-scale neural networks \citep{jacot2018neural,golikov2022neural}. For simplicity, we consider a standard learning problem in regression~\citep{caponnetto2007optimal, liu2023sophia}, where $\mathbf{\Sigma}$ is diagonal with $\mathbf{\Sigma}_{ii} = i^{-a}$ and $\mathbf{\Sigma_{ii}}\mathbf{W}_i = i^{-b}$ with $a, b \geq 1$. Hard problems typically have a smaller $b$. We compare the vanilla SGD without momentum and the accelerated SGD with momentum $1-\beta$. Smaller $\beta$ corresponds to higher momentum, and when $\beta = 1$, accelerated SGD recovers the vanilla one. The comparison is listed in the following theorem.



\begin{theorem}[Validation Loss Rates for SGD and Accelerated SGD]
\label{thm: rate}
Set a constant stepsize of $\eta = \Theta(1)$ and the number of iterations $T$. Then the validation loss of vanilla SGD is bounded by 
$ \tilde{\mathcal{O}}\left( T^{1/a - 1} + T^{1/a - b/a} \right) $. For the accelerated SGD method with momentum $1 - \beta$ (where $\beta \in (0,1]$), the validation loss after $T$ iterations is bounded by $ \tilde{\mathcal{O}}\left( T^{1/a - 1} \beta^{1/a^2 - 1/a} + T^{1/a - b/a} \beta^{\left(1/a^2 - 1/a\right)(1-b)} \right)$.
\end{theorem}

In the context of validation loss, the term $T^{1/a - 1}$ in vanilla SGD and $T^{1/a - 1}\beta^{1/a^2 - 1/a}$ in accelerated SGD correspond to term $\mathcal{A}$ in \eqref{eq:decomposition}, reflecting the validation loss increase due to the solution's variance. In parallel, $T^{1/a - b/a}$ in vanilla SGD and $T^{-b/a + 1/a}\left(\beta^{1/a^2 - a/1}\right)^{-b+1}$ in accelerated SGD capture the deterministic optimization component of the loss and correspond to $\mathcal{B}$. When $a>b$, the term $\mathcal{B}$ dominates the excess risk, and choosing momentum $\beta = T^{(a/b-1)/(1/a-1)}$ yields the smallest upper bound on the excess risk. In contrast, when $a\le b$---an easier regime for deterministic optimization---the variance term $\mathcal{A}$ dominates, and adding momentum (any $\beta\le 1$) increases the final excess risk. This confirms that momentum is not a universal accelerator: in variance-dominated regimes ($a \le b$), it can harm statistical efficiency and even degrade performance.

\subsection{Empirical Insights}
\label{sec:empiricalevidence}

The above analyses illustrate a general theoretical principle on the potential reducibility of momentum. In what follows, we present empirical observations from transformer training that also brings insights into this reducibility. 

\textbf{Sparse Gradients.}
One empirical property of the transformer is the sparsity in gradient matrices. Its existence is demonstrated in Fig.\ref{fig: heatmap} in Appendix~\ref{apx:empir_ill}, where we illustrate the scales of the gradient of the embedding layers and attention projection layers: Most of the columns/rows of the gradient matrices are filled with near-zero values.

The sparse gradients will mitigate the efficacy of momentum, potentially rendering it redundant.  Specifically, the low-frequency gradient signals disrupt gradient accumulation across iterations, and momentum is primarily dominated by single gradients. Consequently, since the architecture units in the transformer optimization exhibit independence \citep{martens2015optimizing}, this invalid accumulation lets the momentum update in a single unit collapse into vanilla gradient descent.

\begin{wrapfigure}[14]{r}{0.48\textwidth}
  \vspace{-5pt}
  \centering
  \includegraphics[width=\linewidth]{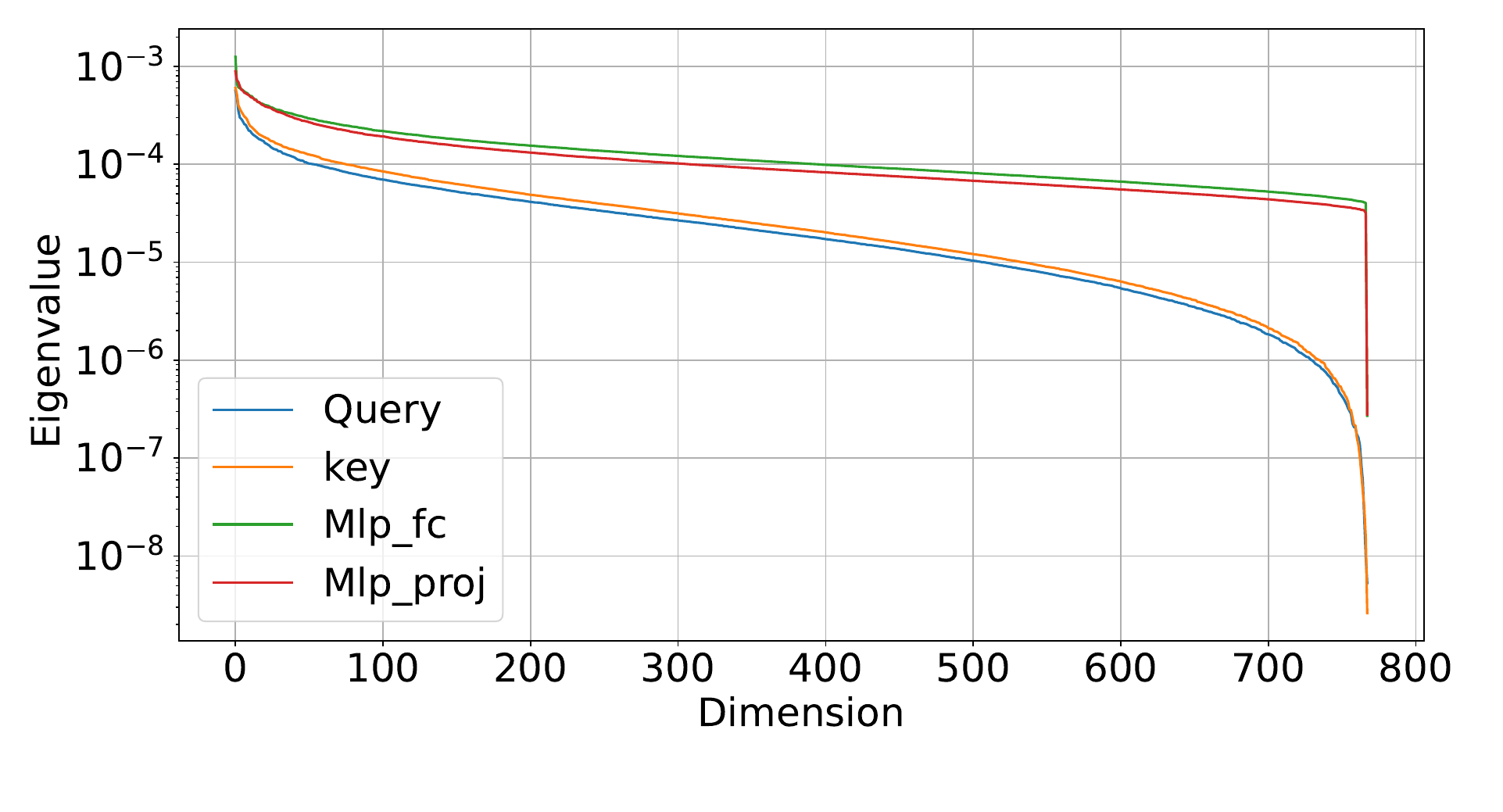}
  \vspace{-20pt}
  \caption{The spectral distribution of features in each block of 10th layer in GPT-2 124M at $10\%$ of the training steps.}
  \label{fig:emp insight}
  \vspace{-10pt}
\end{wrapfigure}
\vspace{-4pt}

\textbf{Gradients Concentrate on a Low-Rank Structure.} Our second observation leading to momentum reducibility is a consistent low effective rank for gradient matrices: few large eigenvalues concentrate in a low-rank subspace, and the following eigenvalues in the tail are extremely small during LLM training. Besides, the low-rank subspace also displays temporal stability, as the dominant singular subspace associated with the low-rank structure changes slowly over time \citep{gur2018gradient}. Fig.\ref{fig:emp insight} illustrates that the energies of the momentum in Query, Key, and MLP blocks concentrate in the top $5\%$ eigenvalues.

This insight has inspired a class of memory-efficient optimization methods, such as LoRA, Galore, and their variants \citep{zhao2024galore, hu2022lora, lialin2023relora}. These techniques constrain parameter updates to a low-rank subspace, thereby significantly reducing memory overhead by discarding residual components. However, a notable limitation arises during pretraining: such methods often underperform compared to full-rank optimization approaches like Adam. The performance gap can be attributed to the non-negligible information loss caused by discarding certain subspaces. Consequently, this naturally raises the open problem: Is it possible to attain the performance of full-rank optimization methods while maintaining the memory of the low-rank structure?

\section{Our Method}
Claims in Section~\ref{sec:starting point} suggest the potential reducibility of momentum in optimization, while in transformer optimization, this reducibility appears less evident. Transformer optimization is neither an easier problem where momentum can be entirely omitted, nor have existing methods effectively exploited low-rank structures to reduce the momentum without degrading performance.

This tension can be addressed by the insight: heterogeneous transformer blocks are better suited to distinct momentum designs. With this insight, we propose our method AdaPM. In AdaPM, instead of treating blocks uniformly, we introduce a non-uniform momentum design for the transformer blocks, tailored to the optimization difficulty. This includes full momentum, no momentum, and low-rank momentum. This partition will be detailed in Section~\ref{sec:nonuniformmomentum}. In parallel, for low-rank momentum, we specifically propose a debiased estimator which empirically { \textbf mitigates} the performance degradation caused by compressing gradients into low-rank structures.

\subsection{Non-Uniform Momentum Design}\label{sec:nonuniformmomentum}


The first component of AdaPM, non-uniform momentum reduction, classifies the reliance on momentum in each transformer block into three distinct regimes: no momentum, full momentum, and low-rank momentum with a debias technique. Empirical results in Table~\ref{tbl:partition} support the division strategy and the heterogeneous momentum requirements: embedding and attention output components exhibit comparable validation loss with and without momentum, whereas even applying a low-rank approximation to the value blocks—in contrast to its effect on other blocks—noticeably slows convergence.

Concretely, our division strategy is illustrated as follows:
\vspace{-0.2em}

\noindent \textbf{(1) Embedding and Attention Output Projection Blocks.} 
Gradients in these transformer components are sparse and lack temporal persistence, rendering momentum of limited value as discussed in Section~\ref{sec:empiricalevidence}. We disable momentum to reduce the momentum overhead.
\vspace{-0.5em}

\noindent \textbf{(2) Query, Key, and MLP blocks.}
These blocks possess more challenging optimization landscapes. Although momentum accelerates convergence, the gradient signal concentrates on a low-rank subspace, leaving the residual with limited information. We therefore adopt a debiased low-rank approximation: (i) compress the momentum via a low-rank projection to reduce momentum; (ii) correct the induced bias by reintroducing a current-iteration residual. This reduces momentum while preserving information that naive truncation would discard, thereby maintaining performance. We will detail this debiased estimator in Section~\ref{sec: debiased}.
\vspace{-0.5em}

\noindent \textbf{(3) Value Blocks.}
For the value layers, however, momentum reduction updates prove inadequate. For value blocks, our method preserves a full momentum update to ensure effective optimization.


Through this non-uniform momentum design, AdaPM achieves a substantial reduction in momentum compared to full momentum methods, such as Adam. Table~\ref{Table: mem} reports the reduction ratio relative to full momentum: an aggregated reduction to merely $5\%$ of the original momentum is guaranteed across models with multiple scales.

\begin{table}[t]
\caption{Comparsion of validation loss with various partition principles on GPT-2 124M. Here Q, K, V, O, Em, stand for Query, Key, Value, and Attention Output, Embedding, respectively, and Full, Low-rank, None stand for Full Momentum, Low-rank Momentum with our correction method, No Momentum, respectively. } 
\label{tbl:partition}
\small
\begin{center}
\begin{tabular}{cccc|cccc}
\hline
 Full  & Low-rank&  None& Final Loss &Full  & Low-rank&  None& Final Loss
\\ \hline 
All blocks &-&-&2.962 &-&Q,K,V, Mlp&Em,  O&3.143$(\mathbf{\uparrow })$\\
Q,K,V, Mlp&-&Em, O&2.964&V&Q,K,  Mlp&Em, O&2.997\\
\end{tabular}
\end{center}
\end{table}

\subsection{Debiased Low-Rank Estimator}
\label{sec: debiased}
For layers utilizing the low-rank momentum, to avoid performance degradation by discarding components outside the low-rank components, we propose a bias-corrected estimation of low-rank momentum to incorporate residual information. Concretely, our low-rank momentum estimation involves two components: (1) low-rank momentum approximation tracking and (2) a bias-correction step.  We first summarize our method for low-rank update in Algorithm~\ref{alg:mars_adamw}.

\begin{algorithm}[t]
\caption{Low Rank Update with Correction for a $m\times n$ layer $\mathbf{W}$}\label{alg:mars_adamw}
\begin{algorithmic}[1]
\Require Weight-decay coefficient $\lambda$, decay rates of momentum $\beta_1, \beta_2$, rank of the momentum approximation matrices $r$ and learning rate schedule $\{\eta_t\}_{t=1}^T$
\State Initialize $\mathbf{L}_0 \in \mathbb{R}^{m\times r}\gets \mathbf{0}$, $\mathbf{R}_0 \in\mathbb{R}^{r\times n}\gets \mathbf{0}$,  $\mathbf{v}_0 \in  \mathbb{R}^{m\times n}\gets \mathbf{0}$ and step $t$ $\gets 0$
\For{$t = 1 \textbf{ to } T$}
    \State Obtain mini-batch gradient $\tilde\nabla f(\mathbf{W}_t,\xi_t)$
    \State  $\mathbf{m}_t \gets (1-\beta_1)\tilde\nabla f(\mathbf{W}_t, \xi_t) + \beta_1 \mathbf{L}_{t-1} \mathbf{R}_{t-1}$
    \State $\mathbf{L}_t, \mathbf{R}_t = \arg\min_{\mathbf{L},\mathbf{R}}\|\mathbf{L}\mathbf{R}-\mathbf{m}_t\|_F^2$ 
    \State $r_t=  \mathbf{L}_t\mathbf{R}_t-\mathbf{m}_t $\Comment{Approximation residual}x
    \State $\mathbf{v}_t = \beta_2 \mathbf{v}_{t-1} + (1-\beta_2) [\tilde\nabla f(\mathbf{W}_t, \xi_t)]^{\odot 2}$\Comment{Standard second-order momentum update}
    \State $\mathbf{m}^{c}_t = \mathbf{m}_t-\frac{\beta_1r_t}{1-\beta_1}$\Comment{Bias correction for low-rank momentum}
    \State $\mathbf{W}_{t+1} = \mathbf{W}_t - \eta_t \left( \operatorname{clip}\left(\frac{\mathbf{m}^c_t}{\sqrt{\mathbf{v}_t}+\epsilon},1\right) + \lambda \mathbf{x}_t \right)$
\EndFor
\end{algorithmic}
\end{algorithm}

The component of tracking low-rank momentum approximations avoids constraining the approximation to a fixed or rarely updated subspace. We incrementally update a low-rank approximation $\mathbf{L}_t\mathbf{R}_t$ of the momentum $\mathbf{M}_t\in\mathbb{R}^{m\times n}$ in each step. Concretely, at iteration $t$, given parameter matrix $\mathbf{W}_t \in \mathbb{R}^{m \times n}$ and the stochastic gradient $\tilde\nabla f(\mathbf{W}_t, \xi_t)$, leveraging the estimate from the previous iteration, $\mathbf{L}_{t-1}\mathbf{R}_{t-1}$, the update is defined by the following optimization problem:
\begin{align}\label{eq:lowrank_update}
\mathbf{L}_t\mathbf{R}_t \in \arg\min_{\mathbf{L},\mathbf{R}}\left\|\mathbf{L}\mathbf{R} -\left((1-\beta_1)\tilde\nabla f(\bf W_t, \xi_t) + \beta_1 \mathbf{L}_{t-1}\mathbf{R}_{t-1} \right)\right\|^2.
\end{align}

The optimization problem in \eqref{eq:lowrank_update} is a standard matrix factorization task, which can be efficiently solved using gradient-based methods \citep{xie2017cumf_sgd}. In implementation, we apply gradient descent warm-starting from the previous estimate $\mathbf{L}_{t-1}\mathbf{R}_{t-1}$. The method typically stabilizes within 5 iterations, yielding accurate low-rank momentum updates with negligible overhead.
The detail of this low-rank approximation can be found at Appendix \ref{sec: lowrank alg}. 

The second component of our method is the bias correction. Directly applying $\mathbf{L}_t\mathbf{R}_t$ as the momentum discards the residual components outside the low-rank structure. Our compensation for the low-rank structure leverages the following one-step residual 
\begin{align}\label{eq:onestepresidual}
r_t = \mathbf{L}_t\mathbf{R}_t - \Big((1-\beta_1)\tilde{\nabla} f(\mathbf{W}_t, \xi_t) + \beta_1 \mathbf{L}_{t-1}\mathbf{R}_{t-1}\Big).
\end{align}
At iteration $t$, $r_t$ denotes the approximation error. Because the $\bf L_{t-1}\bf R_{t-1}$ in the momentum accumulation of equation \ref{eq:lowrank_update} carries forward the residuals from previous steps. Therefore, the bias accumulates as a weighted sum of past errors $r_{t-1}, r_{t-2}, \cdots$. To compensate for the accumulated bias, we refine the momentum estimate by
\begin{align}\label{eq:bias_correction}
\mathbf{m}^c_t = \mathbf{m}_t - \frac{\beta_1r_t}{1-\beta_1}.
\end{align}
where $\frac{\beta_1r_t}{1-\beta_1}$ serves the residual correction to the low-rank approximation. To justify the correction term $\frac{\beta_1 r_t}{1-\beta_1}$ in \eqref{eq:bias_correction}, we assume that the per-iteration residuals are (approximately) stationary.
\begin{assumption}[Stationary Residuals]\label{asp:r_id}
The one-step residuals $\{r_t\}_{t\ge 1}$ are identically distributed across iterations, i.e., $r_t \stackrel{d}{=} r_{t'}$ for all $t,t'\ge 1$. Besides, there exists a constant $C$ such that $\mathbb{E} r_t$ exists and $\|\mathbb{E} r_t\|\leq C$.
\end{assumption}
The near-stationarity of $r_t$ arises from the smoothing induced by the moving average and the incremental updates of $\mathbf{L}, \mathbf{R}$. Moreover, in practice it suffices that residuals are nearly stationary over short horizons $r_t, r_{t+1}, \dots, r_{t+k}$, since the exponential moving average down-weights older terms. {Empirically, this stationary is illustrated as in Fig.\ref{fig: GPT-2 bias}(b), where we observe high consistency of $r_t$s' distributions over windows of $20$ steps.\color{red}}

We compare the compensated momentum $\mathbf{m}_t^{c}$ with the following full-rank momentum:
\begin{align}\label{eq:mfdef}
\mathbf m^f_t = (1-\beta_1)\tilde\nabla f(\mathbf{W}_t,\xi_t) + \beta_1  \mathbf m^f_{t-1}, \quad \text{with} \quad \mathbf{m}^f_0 = \bf 0.
\end{align}
Under Assumption~\ref{asp:r_id}, the following theorem establishes that the proposed debiased momentum precisely eliminates the bias induced by the low rank structure.
\begin{theorem}\label{thm:nobias}
If Assumption~\ref{asp:r_id} holds, then the compensated momentum $\mathbf{m}_t^c$ will asymptotically eliminate the bias in the low-rank estimator $\mathbf{m}^f_t$:
\begin{align*}
 \left\|\mathbb{E}\left[ \mathbf{m}^c_t - \mathbf{m}^f_t\right]\right\| \leq \frac{C}{1-\beta_1}\beta_1^{t+1},
\end{align*}
and therefore $\lim_{t\to\infty} \mathbb{E}[ \mathbf{m}^c_t - \mathbf{m}^f_t] =\bs 0 $.
\end{theorem}

\section{Experiment}
We now validate the effectiveness of AdaPM on both pre-training and fine-tuning tasks. All GPT-2-1.5B experiments were trained on NVIDIA H800-80GB GPUs, while all other models were trained on NVIDIA A6000 GPUs.

\subsection{Pretraining}\label{sec:ex-pre}
\textbf{Setups.} We perform pre-training on the GPT-2 series \citep{brown2020language} (125M to 1.5B parameters) on the OpenWebText \citep{Gokaslan2019OpenWeb} dataset using the nanoGPT implementation. Following standard setting in Adam-mini \citep{zhangadam}, the models are trained with a consistent configuration of 512 batch size, 1024 sequence length, and 0.1 weight decay. We pretrain Llama series (130M to 340M) \citep{touvron2023llama} on C4 \citep{raffel2020exploring}. For all pretraining cases, we apply a cosine learning rate decay (with 2000 warm-up steps) and global gradient clipping at a threshold of 1.0.  We tune the learning rates for all methods and report the curve with the smallest final loss. For AdaPM, we set the same learning rate  as Adam for low rank momentum blocks and full momentum blocks and for blocks without momentum, we set learning rate to $0.75$ times of the learning rate in other blocks. We consistently set $r=5\%$ and $T=100$, and the ablation study is presented in Section \ref{sec: abl}. In addition to AdamW, our evaluation compares AdaPM with several widely-used memory-efficient optimizers: 
\begin{itemize}
\item Adafactor \citep{shazeer2018adafactor}: We incorporate momentum with $\beta_1 =0.9$ to ensure a fair comparison with other methods. We  apply Adafactor with the default hyperparameters: clipping threshold $d=1.0$,$ \epsilon$ =(None, 0.001), $\tau= -0.8$. By tuning the hyperparameters, we set the learning rate to 0.01.
\item Galore \citep{zhao2024galore}:  We set subspace frequency $T$ to 200 and scale
 factor $\alpha$ to 0.25 across all model sizes. We pick the same rank $r=0.5\times \text{dimension}$ while smaller ranks lead to much worse final loss, and we apply them to all multi-head attention layers
 and feed-forward layers in the models.
\item Adam-mini \citep{zhangadam}: We use the same hyperparameter as AdamW, including $\beta_1=0.9$, $\beta_2=0.95$, $\epsilon=10^{-8}$.
\end{itemize} 

\subsubsection{Comparison with Existing Memory-Efficient Optimizers}

 As demonstrated in Fig.\ref{fig: GPT-2}, the loss curves of AdaPM closely resemble those of AdamW in both the GPT-2 series and the Llama series, while alternative methods exhibit slower convergence characteristics. We report the memory cost and GPU hours in Table \ref{Table: mem}, where the batch sizes per GPU are optimized for each algorithm within the GPU memory limits. By implementing a rank reduction to $5\%$ of the original matrix dimensionality, our approach successfully reduces momentum memory consumption to approximately $55\%$ of baseline requirements. Thanks to the memory cut-down, AdaPM can support larger batch sizes per GPU.  We repeated the experiment five times, confirming that our method possesses stability and reproducibility. The results of our study can be replicated, provided the parameter settings are maintained.
\begin{figure}[t]
    \centering
    \subfigure[ GPT-2-124M.]{
    \begin{minipage}{0.46\linewidth}
    \includegraphics[width=70mm]{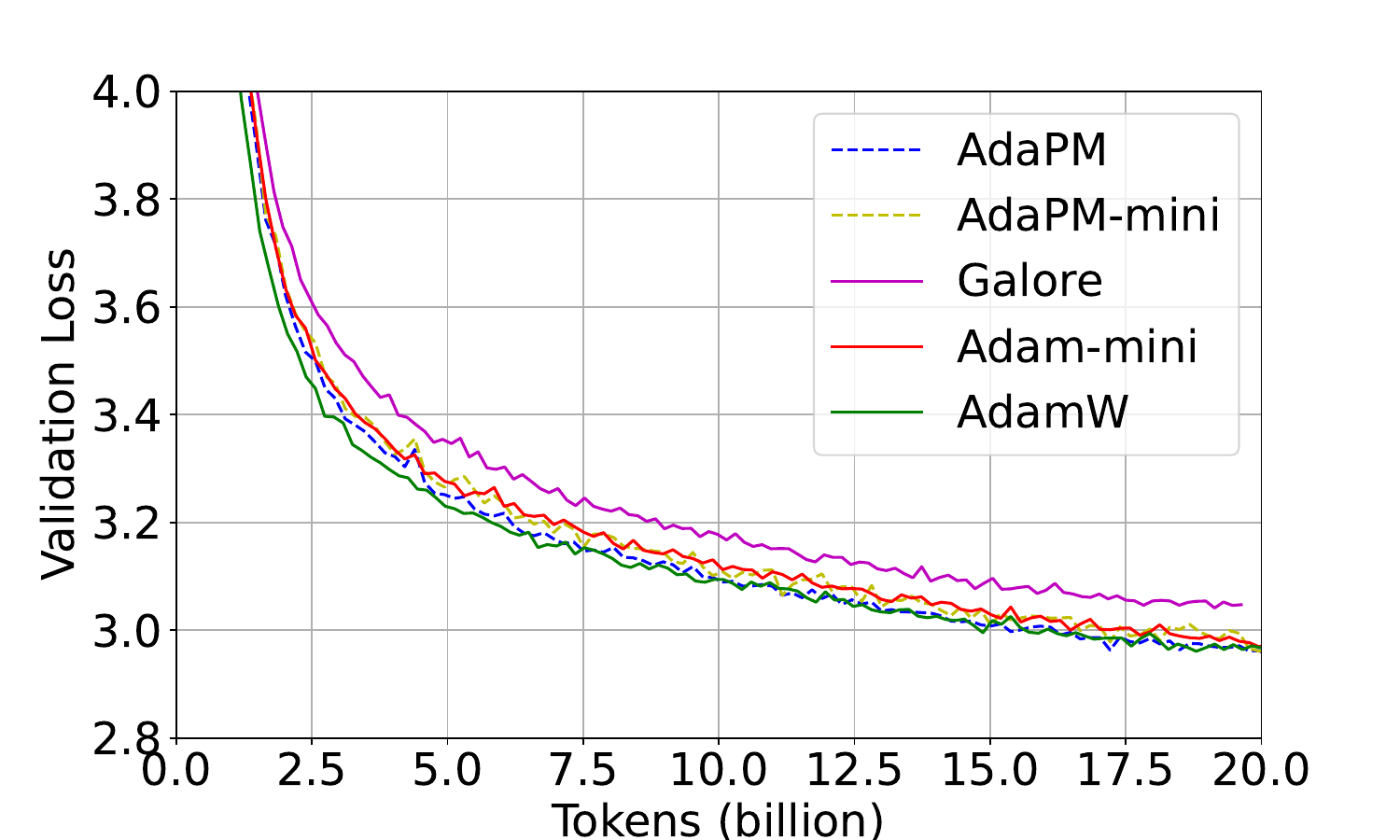}
    \end{minipage}
    
    }
    \subfigure[GPT-2-330M.]{
    \begin{minipage}{0.46\linewidth}
    \includegraphics[width=70mm]{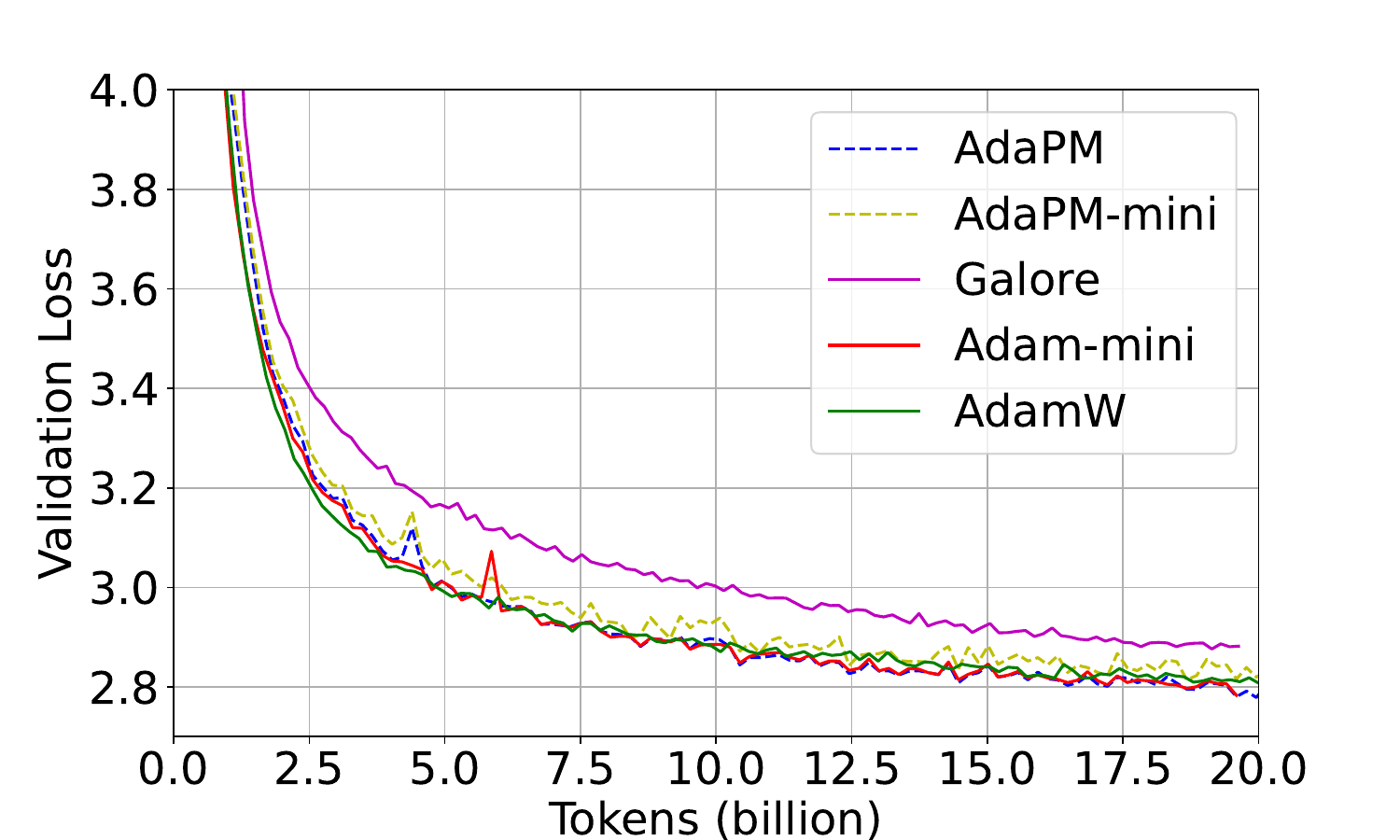}
    \end{minipage}
    }

    \subfigure[Llama-130M.]{
    \begin{minipage}{0.46\linewidth}
  \includegraphics[width=68mm]{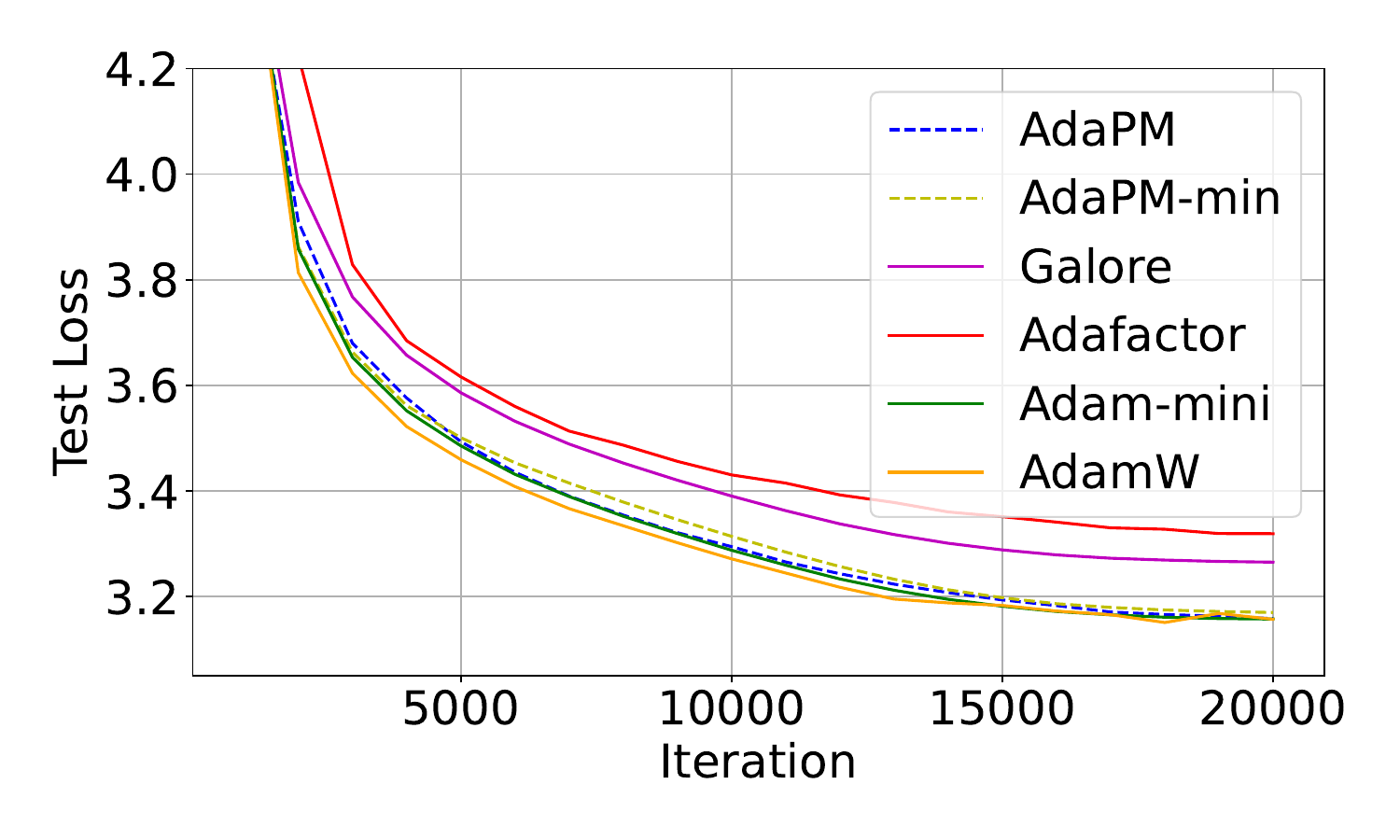}
    \end{minipage}
    }
    \subfigure[Llama-340M.]{
    \begin{minipage}{0.46\linewidth}
    \includegraphics[width=68mm]{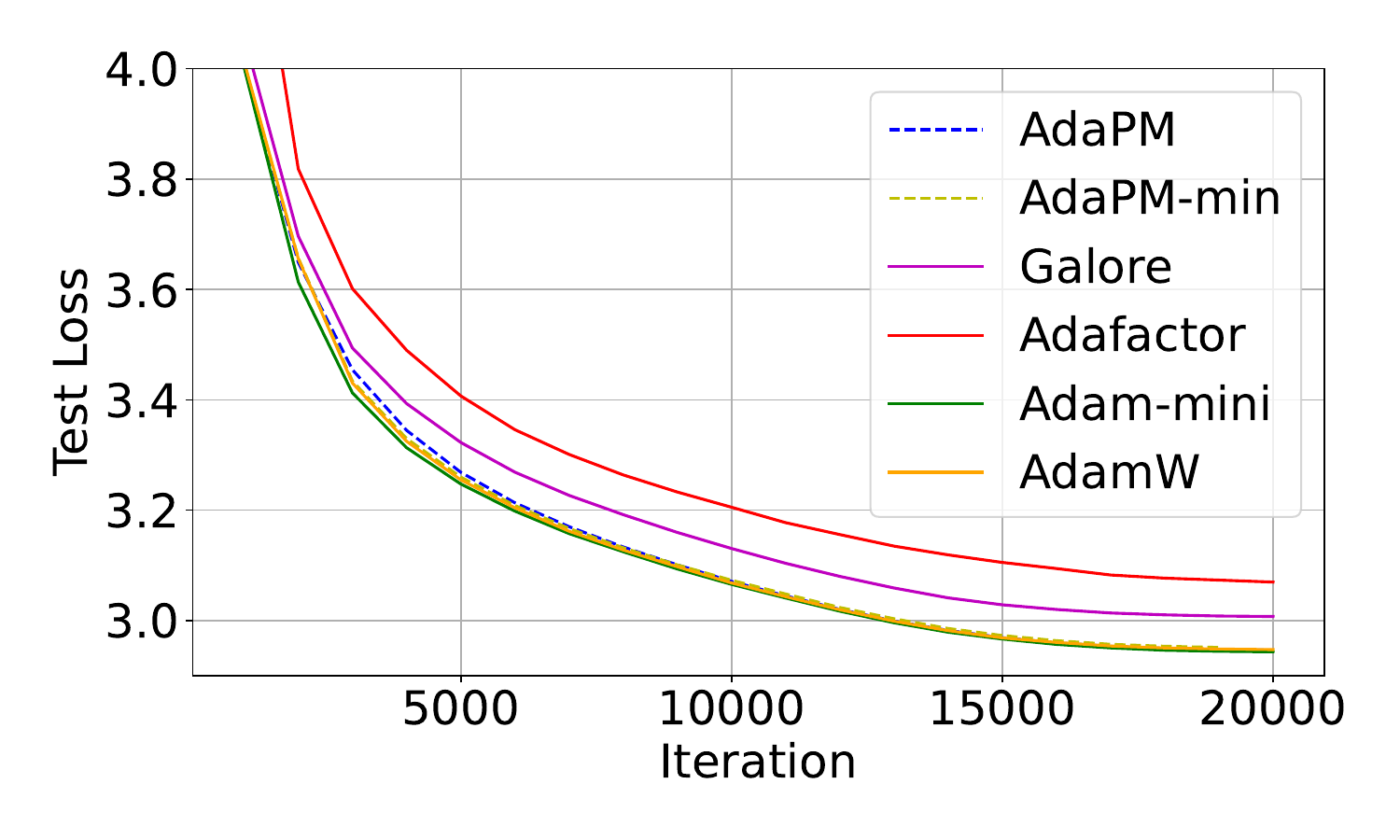}
    \end{minipage}
    }
    \caption{(a)-(b): Loss curves of pre-training GPT-2 series from 124M to 330M. The 1.5B GPT-2 pretrain is in Section~\ref{sec:intro}. (c)(d): Test loss of pre-training Llama-2 series from 130M to 340M. AdaPM performs on par or better than AdamW, while other methods perform worse. }
    \label{fig: GPT-2}
\end{figure}
\begin{table}[t]
\caption{Memory cost of AdamW v.s. AdaPM. Calculation is based on float32, which is a standard choice for optimizer states.}
\label{Table: mem}
\small
\begin{center}
\begin{tabular}{ccc|ccc}
\hline
\multicolumn{3}{c|}{Llama-350M} & \multicolumn{3}{c}{GPT-2-1.5B} \\
\hline
Algorithm & Memory &GPU Hours & Algorithm & Memory &GPU Hours  \\ \hline
Adam & 2.72G & 13.8&Adam & 12.48G &26.67 \\
Adam-mini &1.36G &11.7 & Adam-mini &6.24G &20.32\\
AdaPM &1.48G $(\mathbf{\downarrow 46\%})$ & 11.9$(\mathbf{\downarrow 14\%})$& AdaPM & 6.98G $(\mathbf{\downarrow 44\%})$ &22.11 $(\mathbf{\downarrow 17\%})$\\
AdaPM-mini &0.12G $(\mathbf{\downarrow 96\%})$ &9.8$(\mathbf{\downarrow 29\%})$ & AdaPM-mini & 0.74G $(\mathbf{\downarrow 94\%})$  &17.92 $(\mathbf{\downarrow 33\%})$\\
\end{tabular}
\end{center}
\end{table}


\subsubsection{Ablation Study and Sensitivity Analysis}\label{sec: abl}

We conduct experiments on the GPT-2 124M to evaluate the impact of bias correction in AdaPM. All configurations share identical hyperparameters, with the sole distinction being the inclusion/exclusion of bias correction in low-rank gradient covariance estimation. 

The experimental results in Fig.\ref{fig: GPT-2 bias}(a) demonstrate that omitting bias correction significantly slows down the convergence speed by about 1.96 times even with a rank of $r=50\%$. The convergence speed decreases significantly compared to the bias-corrected method, with the training loss plateauing at higher values throughout the optimization trajectory. Consequently, low-rank approximations without proper bias correction fail to maintain the original model's convergence properties.

\begin{figure}[t]
    \centering
    \subfigure[]{
    \begin{minipage}{0.32\linewidth}
    \includegraphics[width=47.5mm]{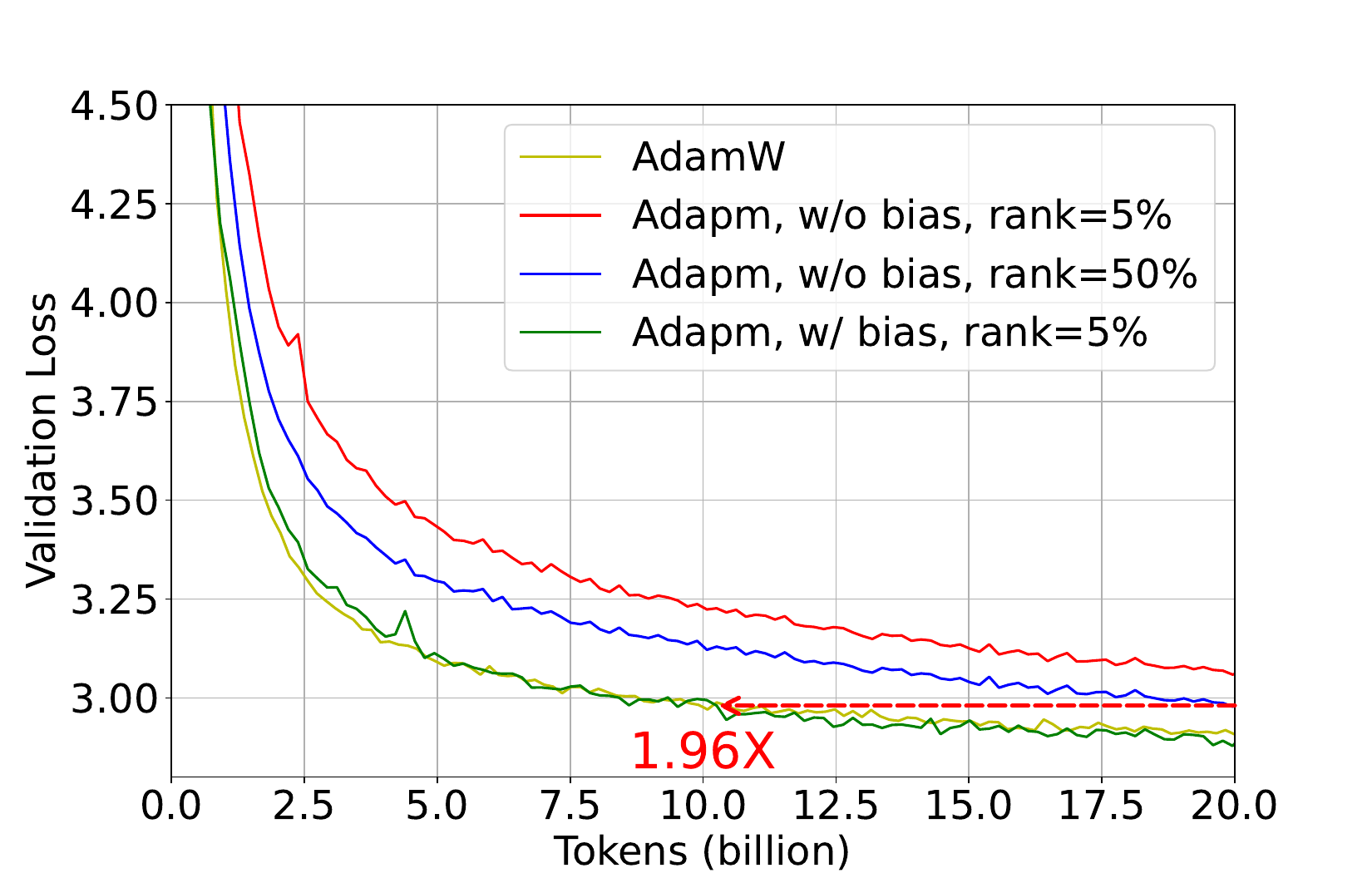}
    \end{minipage}
    }
    \subfigure[]{
    \begin{minipage}{0.31\linewidth}
    \includegraphics[width=49mm]{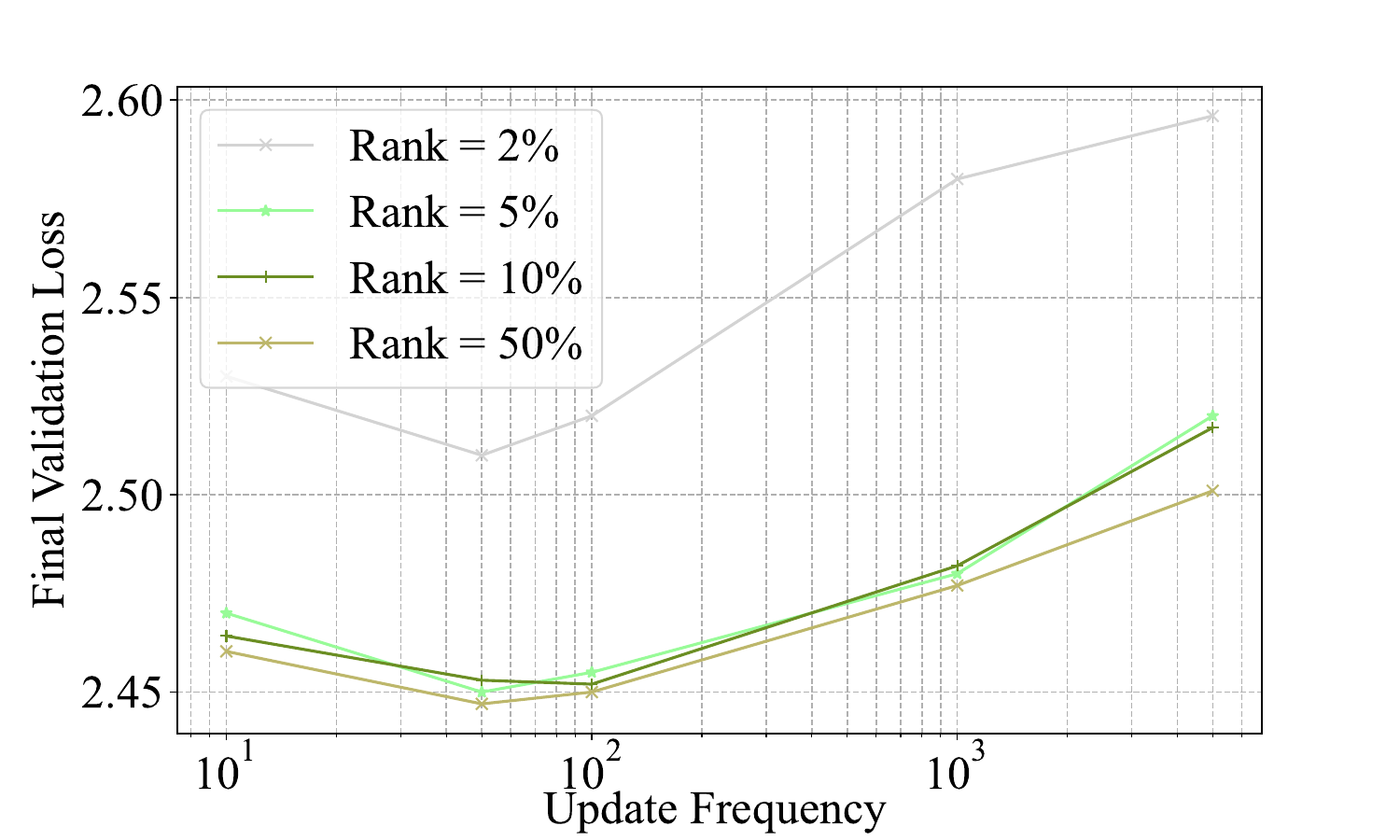}
    \end{minipage}
    }
    \subfigure[]{
    \begin{minipage}{0.32\linewidth}
    \includegraphics[width=47.5mm]{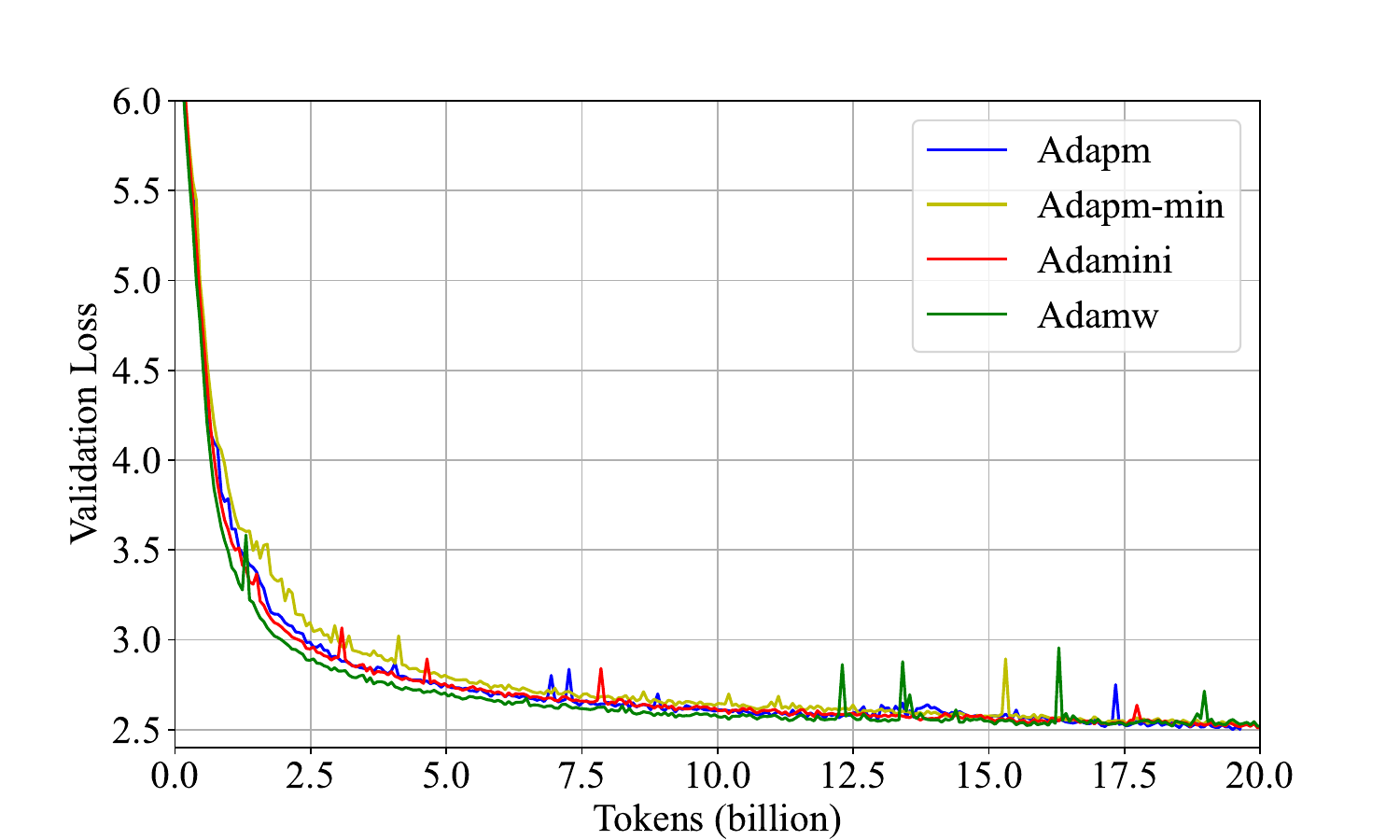}
    \end{minipage}
    }

    \caption{(a) Loss curves of pre-training GPT-2 series with or without bias-correction. (b)Applying AdaPM for pretraining GPT-2-1.5B with different rank and update frequency $T$. (c) Applying AdaPM to Adam-mini for pretraining GPT-2-1.5B.}
    \label{fig: GPT-2 bias}
\end{figure}

We present the final validation loss corresponding to rank ratios of \( r = 2\%, 5\%, 10\%, 50\% \), and update frequencies of \( T = 10, 50, 100, 1000, 5000 \) in Fig.\ref{fig: GPT-2 bias}(b). The closely overlapping trajectories observed across all tested configurations indicate that our algorithm is robust to the choice of rank ratio. Notably, an update frequency of \( T = 100 \) proves sufficient to maintain competitive convergence speed while minimizing computational overhead. This observed stability suggests that even relatively aggressive low-rank approximations—such as those using only 2\% or 5\% of the full rank—can effectively preserve the essential optimization dynamics of the model. Moreover, such approximations yield substantial memory savings, highlighting the practical efficiency and scalability of the proposed method in resource-constrained settings.

\subsubsection{Combining AdaPM with Other First-Order Statistics Reduction Method}
By integrating AdaPM with Adam-mini, we develop a memory-efficient optimization approach that simultaneously reduces the memory footprint of both first-order momentum and second-order variance terms in Adam-type optimizers. As shown in Fig.\ref{fig: GPT-2}(c) and Table \ref{Table: mem}, this combined strategy achieves approximately $95\%$ memory reduction for the optimizer states while maintaining comparable convergence speed to standard Adam.  Experimental results in Fig.\ref{fig: GPT-2 bias}(c) demonstrate that this unified approach maintains model performance on pretraining GPT-2-1.5B while dramatically decreasing memory overhead.

\subsubsection{Scaling Law of AdaPM}
We conduct systematic experiments across the GPT-2 model family (40M, 125M, 350M, 774M, and 1.5B parameters) on the OpenWeb-Text dataset to evaluate AdaPM's scalability. The results plotted in Fig.\ref{fig: RLHF}(a) accord with the log-linear scaling law and suggest this optimization approach remains viable for models beyond 1.5B parameters if the scaling law holds, with projected memory savings of 45\%. Crucially, the convergence stability shows no degradation with increasing model size, indicating that the training dynamics are preserved.

\subsection{LLM Finetuning}

\begin{figure}[tb]
    \centering
    \subfigure[Scaling laws in parameters]{
    \begin{minipage}{0.3\linewidth}
    \includegraphics[width=48mm]{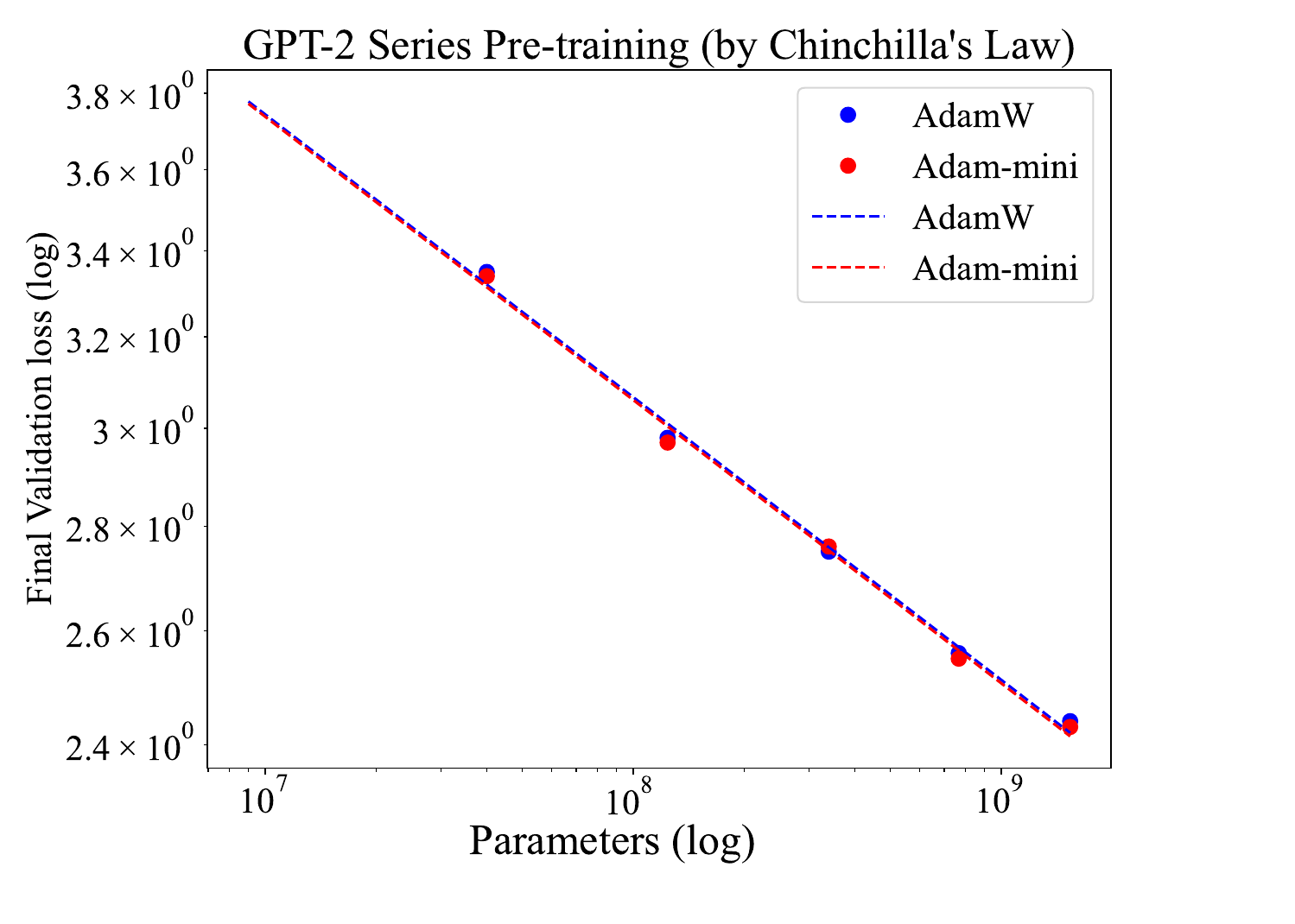}
    \end{minipage}
    }
    \subfigure[SFT]{
    \begin{minipage}{0.32\linewidth}
    \includegraphics[width=46mm]{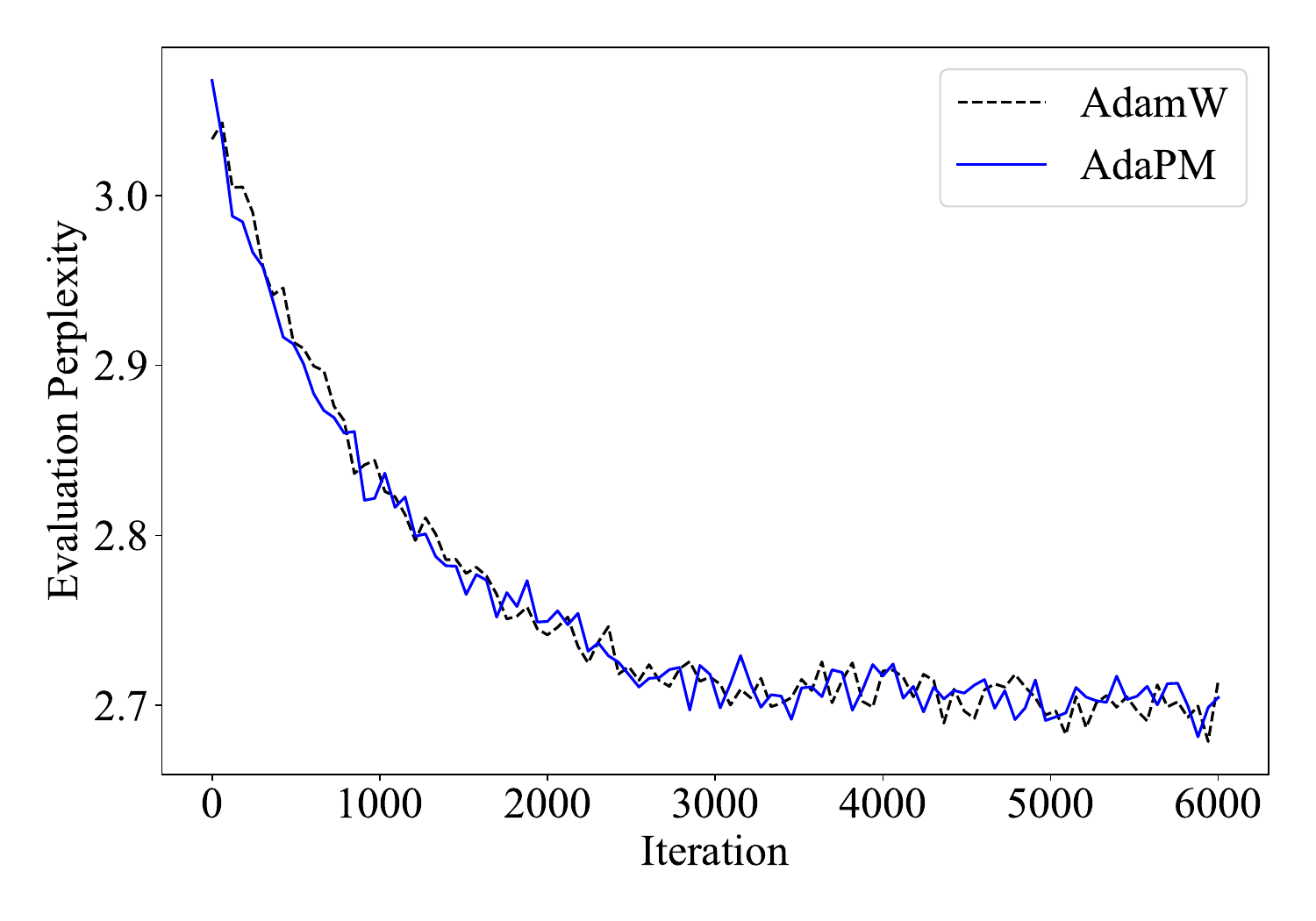}
    \end{minipage}
    }
    \subfigure[RLHF]{
    \begin{minipage}{0.32\linewidth}
    \includegraphics[width=46mm]{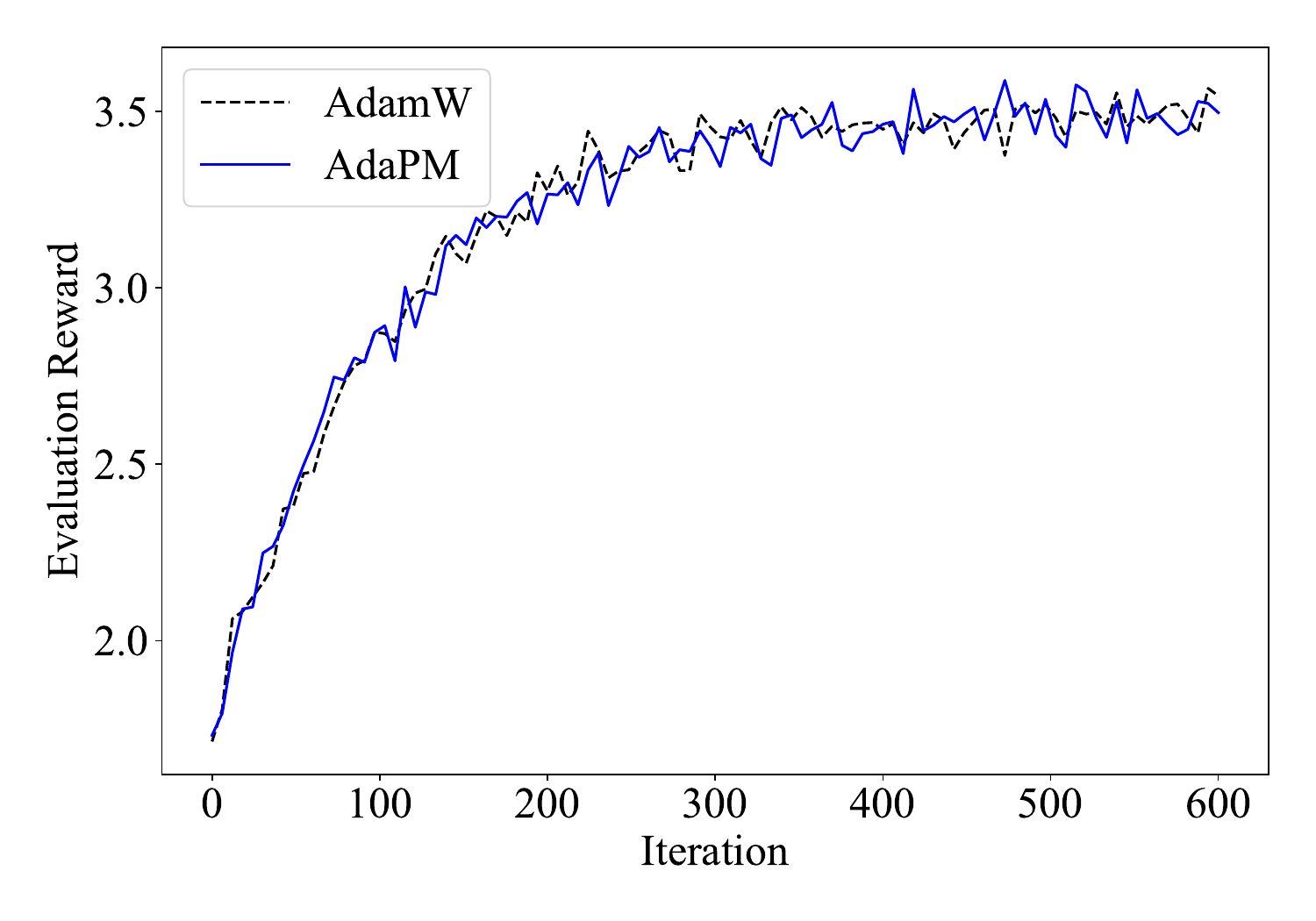}
    \end{minipage}
    }
    \caption{Scaling laws in terms of parameters in (a) suggest that AdaPM can be scaled up to larger models (if the scaling law holds). (b)(c): SFT,and RLHF when aligning Llama3-8B. AdaPM maintains similar evaluation perplexity and reward to AdamW with $43\%$ less memory. }
    \label{fig: RLHF}
\end{figure}

We conducted a comprehensive evaluation of AdaPM on both Supervised Fine-Tuning (SFT) and Reinforcement Learning from Human Feedback (RLHF) tasks. Our experiments were based on the Llama3-8B pretrained model. For the SFT phase, we utilized the UltraFeedback dataset \citep{cui2023ultrafeedback} for training. For the RLHF phase, we implemented the established RLHF pipeline following the methodology described in \citep{ouyang2022training}, with a specific adaptation: we employed ReMax \citep{li2023remax} as our reinforcement learning optimizer. ReMax was chosen as a memory-efficient alternative to the commonly used Proximal Policy Optimization (PPO) algorithm \citep{schulman2017proximal}, which helps in optimizing the policy towards the preference reward model more efficiently.

The results, as illustrated in Fig.\ref{fig: RLHF}, demonstrate that AdaPM delivers performance that is comparable to or surpasses that of the AdamW optimizer across both SFT and RLHF benchmarks. This consistent performance highlights AdaPM's effectiveness and robustness in complex training scenarios, suggesting its potential as a competitive alternative for modern LLM training pipelines.

\section{Conclusion}
We propose AdaPM, an adaptive partial momentum strategy for pretraining and finetuning. Starting from the reducibility of momentum in the Transformer optimizer, AdaPM significantly reduces the momentum demand in the AdamW optimizer through a non-uniform momentum design and a debiasing technique. AdaPM attains a remarkable $94\%$ momentum memory saving on GPT-2 1.5B without sacrificing convergence. It can further reduce memory by up to 95\% in optimizer states by combining the memory-efficient technique on the second-order statistic. There also remains potential for enhancing the design of AdaPM, such as extending the partition method to various other prevalent models \citep{ho2020denoising} and applying the bias-corrected method to the low-rank estimation of activation, which will further reduce the memory cost of training. We leave the development of stronger designs as a future direction.

\bibliography{iclr2026_conference}
\bibliographystyle{iclr2026_conference}
 \newpage
\appendix
\section{Omitted Proofs}

\subsection{Proof of Theorem~\ref{thm: rate}}

Under the setting of Theorem~\ref{thm: rate}, we consider the momentumed ASGD as in Algorithm~\ref{alg:asgd} and the vanilla SGD as in Algorithm~\ref{alg:sgd} \citep{jain2018accelerating}. The step size $\gamma$ of Algorithm~\ref{alg:asgd} is coupled with the momentum parameter $\beta$, which can be found in the parameter choice in \citet{li2023risk}. And in our setting specified in Section~\ref{sec:learningtheory}, $\gamma$ can be decided by $\gamma = \tilde{\Theta}\left( \beta^{-1+1/a} \right)$. Besides, when $\beta = 1$, Algorithm~\ref{alg:asgd} recovers the vanilla SGD algorithm~\ref{alg:sgd}. Both algorithms consider a standard piecewise-constant learning rate schedule commonly used in theoretical analyses \citep{li2023risk}. Specifically, the total training horizon $T$ is divided into phases of length  $K = \lfloor T/\log_2 T \rfloor$, and the stepsizes are updated with 
\begin{align*}
    \delta_t = \frac{\delta_0}{4^{l-1}}, \gamma_t  = \frac{\gamma_0}{4^{l-1}}, \quad \text{if}\ K(l-1)\leq t\ \leq Kl -1.
\end{align*}
Besides, the stochastic gradients are given by $\tilde\nabla f(\mathbf{W}_t,\xi_t) = \left(\langle\mathbf{x}_t, \mathbf{W}_t\rangle - y_t\right) \mathbf{x}_t$.

\begin{algorithm}[t]
    \caption{Accelerated Stochastic Gradient Descent}\label{alg:asgd}
    \begin{algorithmic}
    \Require {Initial weight $\bs W_0$, step sizes $\delta_t$, $\gamma_t$, momentum parameter $\beta$}
    \State $\bs V_0\gets\bs W_0$
    \For{$t=1,2,\ldots,T$}
        \State $\bs U_{t-1}\gets\frac{1}{1+\beta}\bs W_{t-1}+\frac{\beta}{1+\beta}\bs V_{t-1}$
        \State $\bs W_t\gets\bs U_{t-1}-\delta_{t}\tilde\nabla f(\bs W_t, \xi_t)$
        \State $\bs V_t\gets\beta \bs U_{t-1}+(1-\beta)\bs V_{t-1}-\gamma_{t} \tilde\nabla f(\bs W_t, \xi_t)$
        \EndFor
    \end{algorithmic}
\end{algorithm}

\begin{algorithm}[t]
    \caption{Vanilla Stochastic Gradient Descent (SGD)}\label{alg:sgd}
    \begin{algorithmic}
    \Require {Initial weight $\bs W_0$, step sizes $\delta_t$}
    \For{$t=1,2,\ldots,T$}
        \State $\bs W_t \gets \bs W_{t-1} - \delta_t \tilde{\nabla} f(\bs W_t, \xi_t)$
        \EndFor
    \end{algorithmic}
\end{algorithm}

The result in Theorem~\ref{thm: rate} is adapted from the theorem in \citet{liu2025optimal} to our setup.
\begin{theorem}[Upper Bound of Accelerated SGD]\label{thm:asgd-upper-bound}
Let $\mathcal{W}_{\mathbf{M}}=\left\{\bs W^*\in\bbR^d:\left\Vert\bs W^*\right\Vert_{\bs M}^2\leq 1\right\}$. For the positive semi-definite diagonal matrix $\bs \Sigma$ and samples drawn from the distribution specified in Section~\ref{sec:learningtheory}, the excess risk of $\bs W_T$ from Algorithm~\ref{alg:asgd} can be bounded from above by
    \begin{equation}\label{eq: asgdupper}
    \begin{aligned}        
        \bbE\left\Vert\bs W_T-\bs W^*\right\Vert_{\bs\Sigma}^2\leq&\underbrace{\left(\sigma^2+2c\right)\cdot\left[\frac{k^*}{T}+T\left(\gamma+\delta\right)^2\sum_{i=k^*+1}^{d}\bs\Sigma_{ii}^2\right]}_\text{Effective Variance}\\
        &+\underbrace{\frac{\left\Vert\bs \Sigma'_{0:k^*}\right\Vert}{8T^2(\log_2 T)^4}+4\left\Vert\bs \Sigma'_{k^*:\infty}\right\Vert}_\text{Effective Bias}, \\
    \end{aligned}
    \end{equation}
    where $k^*=\max\left\{k:\bs\Sigma_{kk}>\frac{32\ln n}{(\gamma+\delta)T}\right\}$, $\bs \Sigma'=\bs M^{-1/2}\bs \Sigma\bs M^{-1/2}$, $\bs 
    \Sigma_{0:k^*}'=\bs M^{-1/2}\bs \Sigma_{0:k^*}\bs M^{-1/2}$ and $\bs 
    \Sigma_{k^*:\infty}'=\bs M^{-1/2}\bs \Sigma_{k^*:\infty}\bs M^{-1/2}$.
\end{theorem}
\begin{proof}
In the setup of diagonal matrix $\bs \Sigma_ii = i^{-a}$, and $\bs\Sigma_{ii}\bs W_{i}^2 = i^{-b}$, we specify $\bs M$ in Theorem~\ref{thm:asgd-upper-bound} to be diagonal matrix with $\bs M_{ii} = \tilde{\mathcal{O}} \left(i^{b-a-1}\right)$, which suffices to make $\bs W^* \in \mathcal{W}_{\mathbf{M}}$.

Given the momentum $1-\beta$, the corresponding $\gamma$ is given by $\gamma = \tilde\Theta\left( \beta^{-1+1/a} \right)$, and therefore, the corresponding $k^*$ is given by $k^* = \tilde{\Theta}\left( 1/(T\gamma) \right)^{-1/a}$.

Plugging $k^*, \gamma, \mathbf{M}$ into the upper bound \ref{eq: asgdupper}, the effective variance will be bounded by $ \tilde{\mathcal{O}}\left( T^{1/a - 1} \beta^{1/a^2 - 1/a}  \right)$ and the effective bias can ben bounded by $ \tilde{\mathcal{O}}\left(  T^{1/a - b/a} \beta^{\left(1/a^2 - 1/a\right)(1-b)} \right)$, which collectively completes the bound for momentum algorithm in Theorem~\ref{thm: rate}. Furhter, setting $\beta=1$ and the momentumed SGD in Algorithm~\ref{alg:asgd} degrades to the vanilla SGD in Algorithm~\ref{alg:sgd}, and we obtain the excess risk of $ \tilde{\mathcal{O}}\left( T^{1/a - 1} + T^{1/a - b/a} \right) $ claimed in Theorem~\ref{thm: rate}.
\end{proof}

\subsection{Proof of Theorem~\ref{thm:nobias}}
\begin{proof}
First, by the definition of $\mathbf{m}_t$ in Algorithm~\ref{alg:mars_adamw} and $\mathbf{m}_t^f$ in \eqref{eq:mfdef}, we have
\begin{align}\label{eq:iter}
\begin{aligned}
    \mathbf{m}_t - \mathbf{m}_t^f = & (1-\beta_1)\tilde\nabla f(\mathbf{W}_t,\xi_t) + \beta_1 \mathbf{L}_{t-1}\mathbf{R}_{t-1} - (1-\beta_1)\tilde\nabla f(\mathbf{W}_t,\xi_t) - \beta_1 \mathbf{m}_{t-1}^f\\
    = & \beta_1r_{t-1} + \beta_1\mathbf{m}_{t-1} - \beta_1 \mathbf{m}_{t-1}^f = \beta_1(\mathbf{m}_{t-1} - \mathbf{m}_{t-1}^f) + \beta_1 r_{t-1}.
\end{aligned}
\end{align}
For any given $k$, by telescoping the recursion in \eqref{eq:iter}, we obtain
\begin{align*}
\mathbb{E}\left[\mathbf{m}_t - \mathbf{m}_t^f\right] = \beta^k\mathbb{E}\left(\mathbf{m}_{t-k} - \mathbf{m}_{t-k}^f\right) + \sum_{i = 1}^{k}\beta^{i}\mathbb{E}r_{t-i} = \beta^k\mathbb{E}\left(\mathbf{m}_{t-k} - \mathbf{m}_{t-k}^f\right) + \frac{\beta_1 - \beta_1^{k+1}}{1- \beta_1}\mathbb{E} r_{t},
\end{align*}
where the last equality follows from Assumption~\ref{asp:r_id}. Consequently, our compensated momentum satisfies
\begin{align*}
    \mathbb{E}\left[ \mathbf{m}_{t}^c - \mathbf{m}_t^f \right] = & \mathbb{E}\left[\mathbf{m}_t - \mathbf{m}_t^f \right] - \frac{\beta_1}{1-\beta_1}\mathbb{E}r_t\\
    = & \beta^k\mathbb{E}\left(\mathbf{m}_{t-k} - \mathbf{m}_{t-k}^f\right) - \frac{ \beta_1^{k+1}}{1- \beta_1}\mathbb{E} r_{t}
\end{align*}

When $k = t$, $\mathbf{m}^f_0 = \bs0$, and $\mathbf{m}_0 = \mathbf{L}_0\mathbf{R}_0 = \bs 0$, it follows that
\begin{align*}
\left\|\mathbb{E}\left[ \mathbf{m}_{t}^c - \mathbf{m}_t^f \right]\right\| =  \frac{ \beta_1^{t+1}}{1- \beta_1} \|\mathbb{E} r_{t}\| \leq \frac{C}{1-\beta_1}\beta_1^{t+1},
\end{align*}
where the second inequality follows the upper bound of $\|\mathbb{E}r_t\|$ in Assumption~\ref{asp:r_id}. This completes the proof of Theorem~\ref{thm:nobias}.

\end{proof}
\section{Algorithm for Solving Optimization Problems in \eqref{eq:lowrank_update}}
\label{sec: lowrank alg}

\begin{algorithm}[t]
\caption{Low Rank Approximation for a $m\times n$ momentum $m_t$}\label{alg:lowrank approximation}
\begin{algorithmic}[1]
\Require Iteration number $K$, low rank approximation in the last iteration $\mathbf{L}_{t-1}, \mathbf{R}_{t-1}$, rank of the momentum approximation matrices $r$ and learning rate schedule $\{\eta_k^L\}_{k=1}^K, \{\eta_k^R\}_{k=1}^K$
\State Initialize $\mathbf{L} \in \mathbb{R}^{m\times r}\gets \mathbf{L}_{t-1}$, $\mathbf{R} \in\mathbb{R}^{r\times n}\gets \mathbf{R}_{t-1}$
\For{$k = 1 \textbf{ to } K$}
    \State Calculate the gradient
    $$\frac{\partial l}{\partial\mathbf{L}}=(\mathbf{L}\mathbf{R}-\mathbf{m}_t)\mathbf{R}^\top, \quad \frac{\partial l}{\partial\mathbf{R}}=\mathbf{L}^\top(\mathbf{L}\mathbf{R}-\mathbf{m}_t)$$
    \State Update
    $\mathbf{L}=\mathbf{L}-\eta_k^L\frac{\partial l}{\partial\mathbf{L}}$
    \State Update
    $\mathbf{R}=\mathbf{R}-\eta_k^R\frac{\partial l}{\partial\mathbf{R}}$
\EndFor\\
\Return $ \mathbf{L},\mathbf{R}$
\end{algorithmic}
\end{algorithm}

The optimization problem in \eqref{eq:lowrank_update}  can be efficiently solved using gradient-based methods \citep{xie2017cumf_sgd}. The implementation is shown in Algorithm \ref{alg:lowrank approximation}, where we apply gradient descent warm-starting from the previous estimate $\mathbf{L}_{t-1}\mathbf{R}_{t-1}$. We also employ a learning rate schedule of $\eta_k^L=\eta_k^R=0.5*\left( 1+\cos \left(\frac{\pi k}{K}\right)\right)$. We set the iteration number $K=5$ which is enough for the method to be stable, yielding accurate low-rank momentum updates with negligible overhead. 
\newpage
\section{Omitted Empirical Illustration on Sparsity}\label{apx:empir_ill}
Fig.\ref{fig: heatmap} illustrates the sparsity of the gradient matrices of major transformer blocks.
\begin{figure}[]
    \centering
    \subfigure[Embedding.]{
    \begin{minipage}{0.29\linewidth}
    \includegraphics[width=57mm]{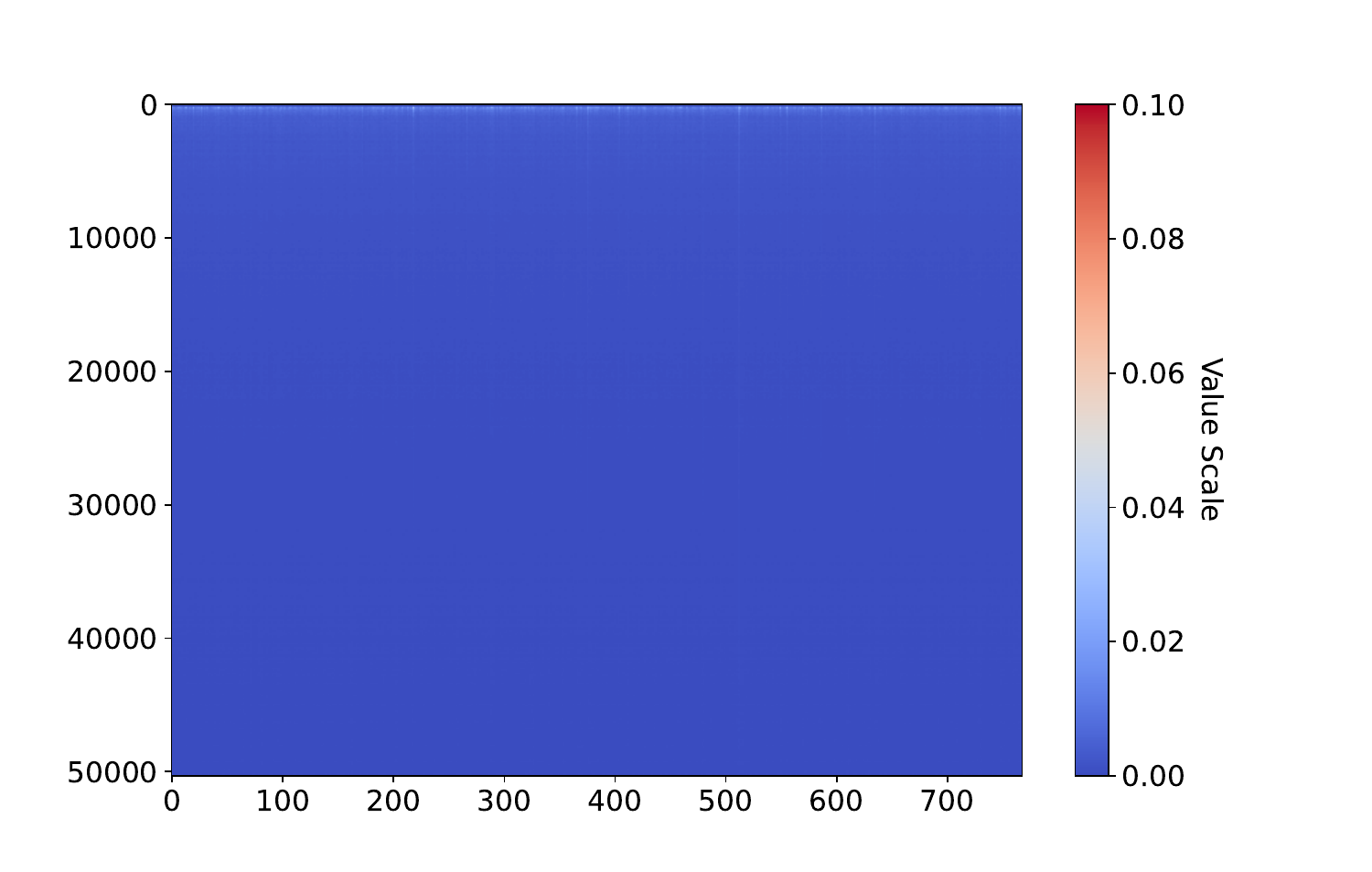}
    \end{minipage}
    }
    \subfigure[Query.]{
    \begin{minipage}{0.29\linewidth}
    \includegraphics[width=57mm]{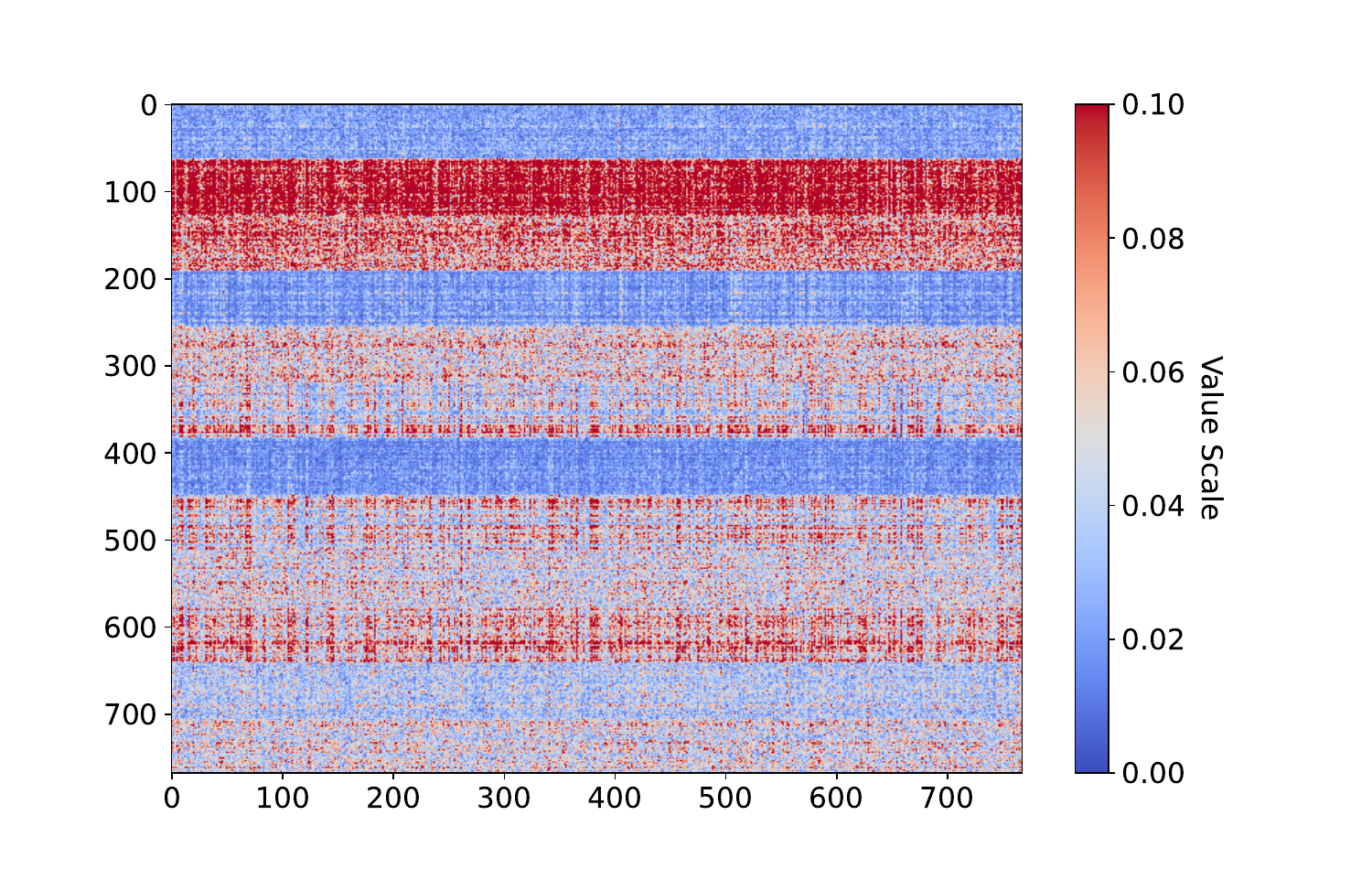}
    \end{minipage}
    }
    \subfigure[Key.]{
    \begin{minipage}{0.29\linewidth}
    \includegraphics[width=49mm]{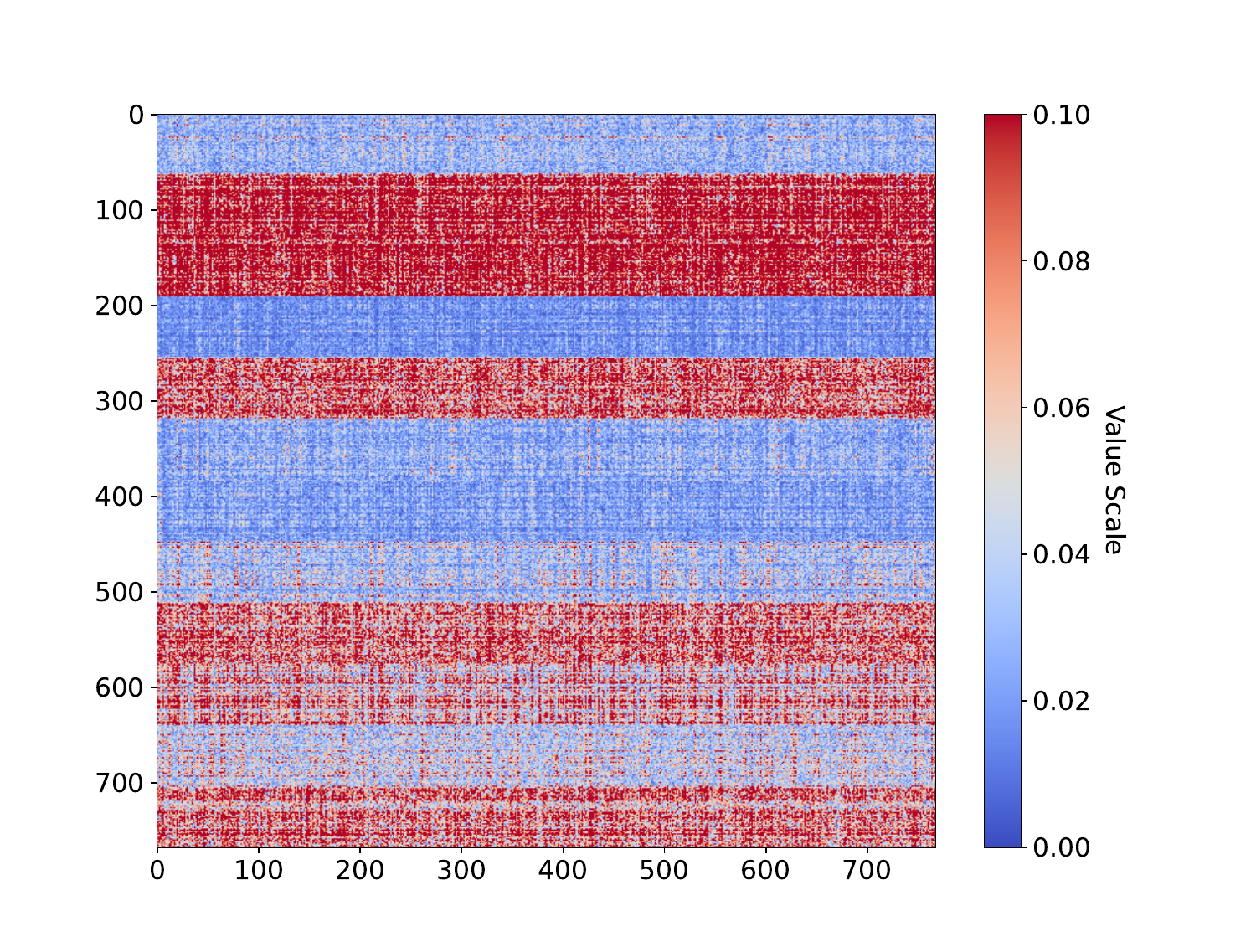}
    \end{minipage}
    }

    \subfigure[Value.]{
    \begin{minipage}{0.29\linewidth}
    \includegraphics[width=54mm]{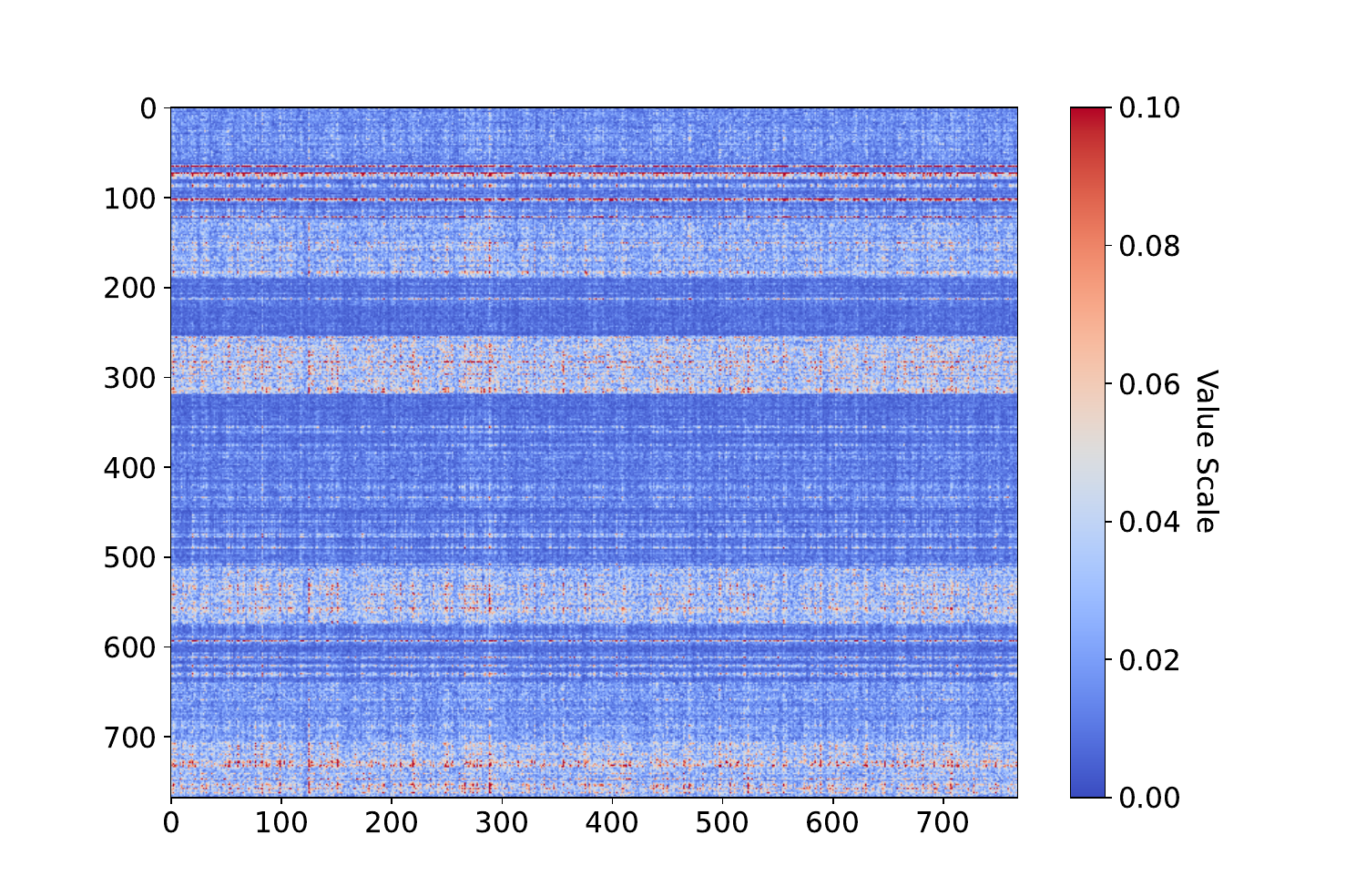}
    \end{minipage}
    } \hspace{-3mm}
    \subfigure[Attn.proj.]{
    \begin{minipage}{0.3\linewidth}
    \includegraphics[width=54mm]{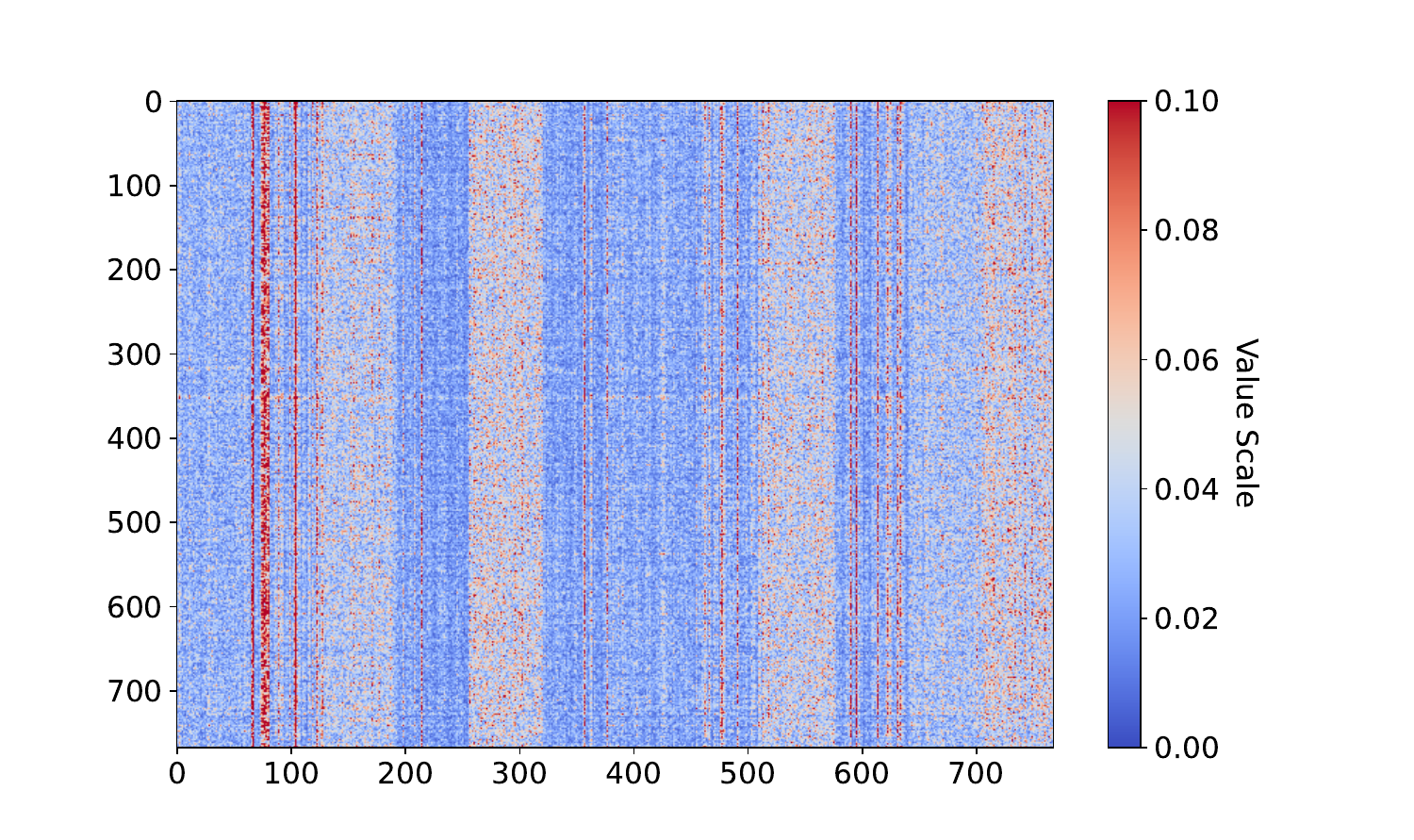}
    \end{minipage}
    }\hspace{-3mm}
    \subfigure[mlp.proj.]{
    \begin{minipage}{0.28\linewidth}
    \includegraphics[width=62mm]{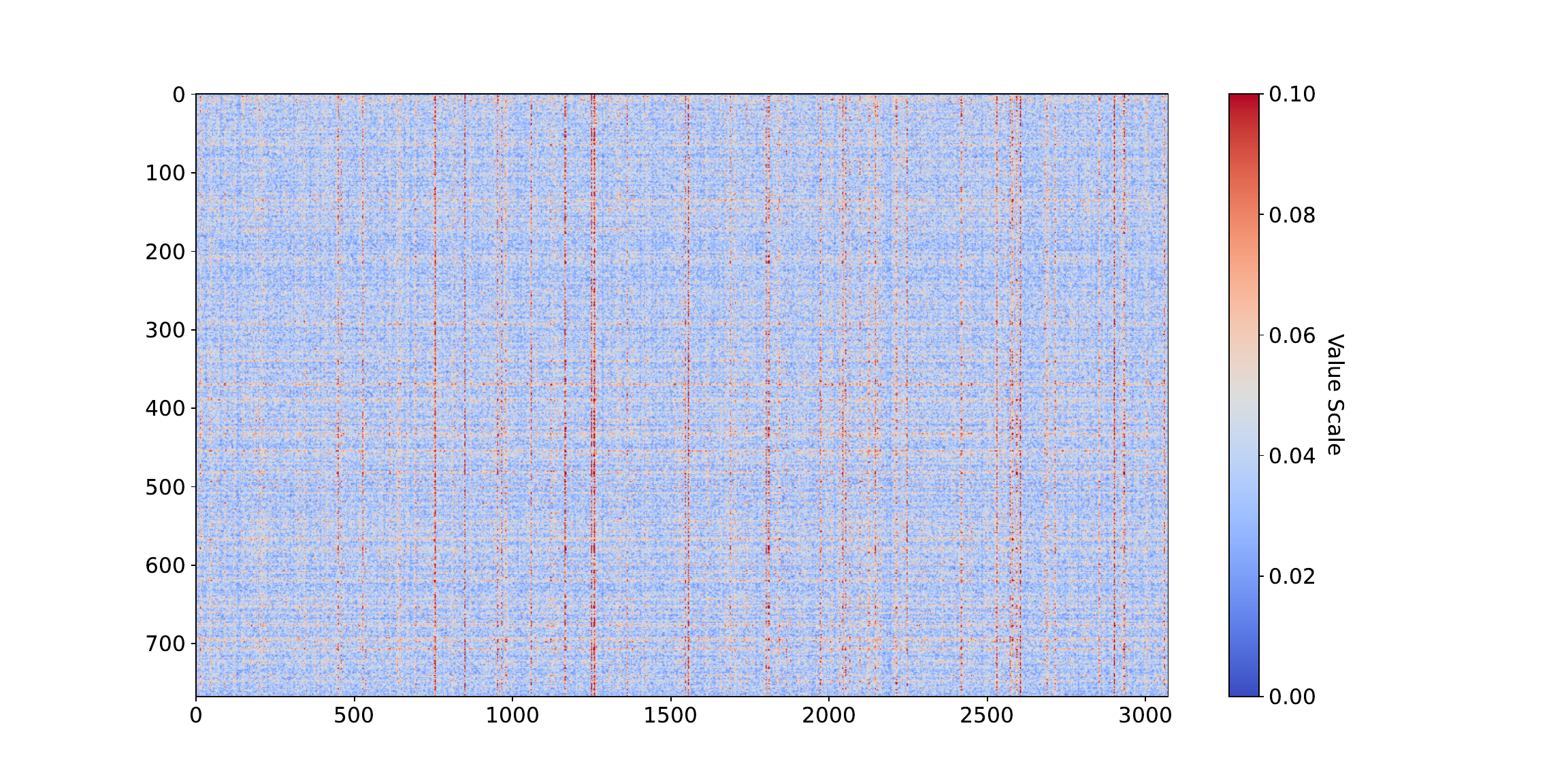}
    \end{minipage}
    }
    \caption{Heatmap of the ratio between gradients and the maximum value in gradient matrices in GPT-2 124M at $10\%$ training step.}
    \label{fig: heatmap}
\end{figure}

\end{document}

%% file: math_commands.tex
\usepackage{amsmath,amsfonts,bm}

\def\eqref#1{equation~\ref{#1}}

\def\1{\bm{1}}

\DeclareMathAlphabet{\mathsfit}{\encodingdefault}{\sfdefault}{m}{sl}
\SetMathAlphabet{\mathsfit}{bold}{\encodingdefault}{\sfdefault}{bx}{n}